\documentclass[10pt,journal,compsoc]{IEEEtran}
\usepackage{bm}
\usepackage{amsmath}
\usepackage{amssymb}
\usepackage{amsthm}
\usepackage{latexsym}
\usepackage{bbding}
\usepackage{algorithm}
\usepackage{algorithmic}
\usepackage{amsmath}
\usepackage{array}
\usepackage{booktabs}
 \usepackage{dblfloatfix}
 \usepackage{fixltx2e}
\usepackage{multirow}
\usepackage{graphicx}
\usepackage{subfigure}
\usepackage{url}
\usepackage{authblk}
\usepackage{array}
\usepackage{booktabs}
\usepackage{xcolor}
\usepackage[T1]{fontenc}
\usepackage{inconsolata}

\newtheorem{theorem}{\bf Theorem}
\newtheorem{lemma}{\bf Lemma}

\newtheorem{corollary}{\bf Corollary}

\ifCLASSOPTIONcompsoc
  \usepackage[nocompress]{cite}
\else
  \usepackage{cite}
\fi

\ifCLASSOPTIONcompsoc
  \usepackage[caption=false,font=footnotesize,labelfont=sf,textfont=sf]{subfig}
\else
  \usepackage[caption=false,font=footnotesize]{subfig}
\fi

\ifCLASSOPTIONcaptionsoff
  \usepackage[nomarkers]{endfloat}
\fi
\begin{document}

\title{Privacy Threats Analysis to \\ Secure Federated Learning}

\author{Yuchen Li, Yifan Bao, Liyao Xiang,~\IEEEmembership{Member,~IEEE,} Junhan Liu, Cen Chen, Li Wang, Xinbing Wang,~\IEEEmembership{Senior~Member,~IEEE}
	\thanks{Y. Li, Y. Bao, L. Xiang (the corresponding author, xiangliyao08@sjtu.edu.cn), J. Liu are with Shanghai Jiao Tong University, China. C. Chen and L. Wang are with Ant Group. This work was partially supported by NSF China (61902245, 62032020, 61960206002), CCF-Ant Group Research Fund (CCF-AFSG RF20200007) and the Science and Technology Innovation Program of Shanghai (19YF1424500).}
}

\markboth{Journal of \LaTeX\ Class Files,~Vol.~14, No.~8, August~2015}%
{Shell \MakeLowercase{\textit{et al.}}: Privacy Threats Analysis to Secure Federated Learning}

\IEEEtitleabstractindextext{%
\begin{abstract}
Federated learning is emerging as a machine learning technique that trains a model across multiple decentralized parties. It is renowned for preserving privacy as the data never leaves the computational devices, and recent approaches further enhance its privacy by hiding messages transferred in encryption. However, we found that despite the efforts, federated learning remains privacy-threatening, due to its interactive nature across different parties. In this paper, we analyze the privacy threats in industrial-level federated learning frameworks with secure computation, and reveal such threats widely exist in typical machine learning models such as linear regression, logistic regression and decision tree. For the linear and logistic regression, we show through theoretical analysis that it is possible for the attacker to invert the entire private input of the victim, given very few information. For the decision tree model, we launch an attack to infer the range of victim's private inputs. All attacks are evaluated on popular federated learning frameworks and real-world datasets.
\end{abstract}

}

\maketitle

\IEEEdisplaynontitleabstractindextext

\IEEEpeerreviewmaketitle

\section{Introduction}
\label{sec:intro}

Proposed by Google, federated learning (FL) enables computational parties to collaboratively learn a shared model while keeping all training data local, decoupling the ability to perform machine learning from the need to store the data in one place. The approach stands in contrast to traditional centralized machine learning where local datasets are collected at one central server, and has gradually become the gold standard for proprietary entities to learn jointly across their information boundaries. Federated learning frameworks such as {\tt TensorFlow Federated, FATE, PySyft} and others have grown more and more popular in the production line.

Depending on the distribution characteristics of the data, FL frameworks work differently. Horizontal Federated Learning (HFL) jointly trains a model over distributed data samples sharing the same attributes. For example, demographic profiles of different regions share the same attributes but totally different data records.  Vertical Federated Learning (VFL) learns the global model on data separated by attributes. An example is that, commercial banks, revenue agencies, or third-party payment platforms may own different perspectives of a client's financial well-being, and are required to jointly analyze the customer behavior. In both HFL and VFL settings, distributed entities cannot share their proprietary data with others.

Although private input stays local, it is recognized that FL reveals much private information through intermediate results. By inverting from the gradients or features, it is possible to recover the original input by recent studies \cite{zhu2019deep, inproceedings1}. Beyond inversion attacks, property inference attacks are capable of inferring private input attributes. Observing the privacy leakage, a variety of works propose secure FL, {\em i.e.,} homomorphic encryption (HE), multi-party computation (MPC) and other cryptography-based techniques, to secure the process. With secure computation protocols implemented, it is hard for any party to obtain the original gradients or features in plain text. 

However, we found that privacy threats remain at large even with secure FL. Attackers can infer unexpected information from the interactions with the victim, and hence pose as a significant privacy threat. Previous studies \cite{hitaj2017deep, nasr2019comprehensive, melis2019exploiting} assume the attacker plays maliciously to inject intentionally designed inputs, for the purpose of tricking the victim to releasing more private information. We do not assume such an attack as the behavior would harm the shared model, but rather assume an honest-but-curious attacker who follows the secure FL protocol and only exploits what it obtains during the process. We aim to give a quantified analysis of the privacy risk in secure FL. 

We analyze the privacy leakage of linear and logistic regression models in the two-party as well as the multi-party FL settings. In these settings, we show the privacy leakage by constructing inversion attacks on victim's private training data. Taking advantage of the encrypted gradients exchanged at each iteration, we demonstrate it is possible for the attacker to extract quadratic equations on the private training data of the victim's. If we allow the attacker to send arbitrary queries at the inference phase, linear equations of victims' private training data can be further obtained. With the knowledge combined, we illustrate how easy it is for the attacker to invert all private training data of the victim's.

We further analyze the privacy risk of SecureBoost model in FL setting. In SecureBoost, the tree model classifies data samples by their attribute ranges without revealing the attribute. However, we discover that, with a careful design of the input query, it is highly likely for an attacker to recover a precise range of victim's private data, posing great threats to the model.

Our major contribution can be considered as an analysis on how easy it is to invert a participant's private inputs in FL settings. As we found out, the system is not as secure as it typically claims, as the attacker can invert all victim's private data given very limited information at a few spots. We analyze such risk by formulating equations and seeking the necessary and/or sufficient conditions that the equations having finite solutions. With rigorous proofs, we show the attacker quickly drills down to a few solutions given only limited information.  This indicates participants' private training data is potentially at significant risks, even in the secure FL.

Highlights of our work include: by analyzing the participants' interactions in secure FL settings, we point out that privacy threats widely exist in regression models and tree models. We quantify such privacy threats by an equation system, based on which we successfully launch attacks on industrial-level systems in experiments.

\section{Background}
\label{sec:background}
We introduce the common secure federated learning frameworks. For each of the targeted scenarios, we first provide the general setting and give a specific framework as an example.

\begin{figure}[th]
	\centering
	\includegraphics[width=0.75\linewidth]{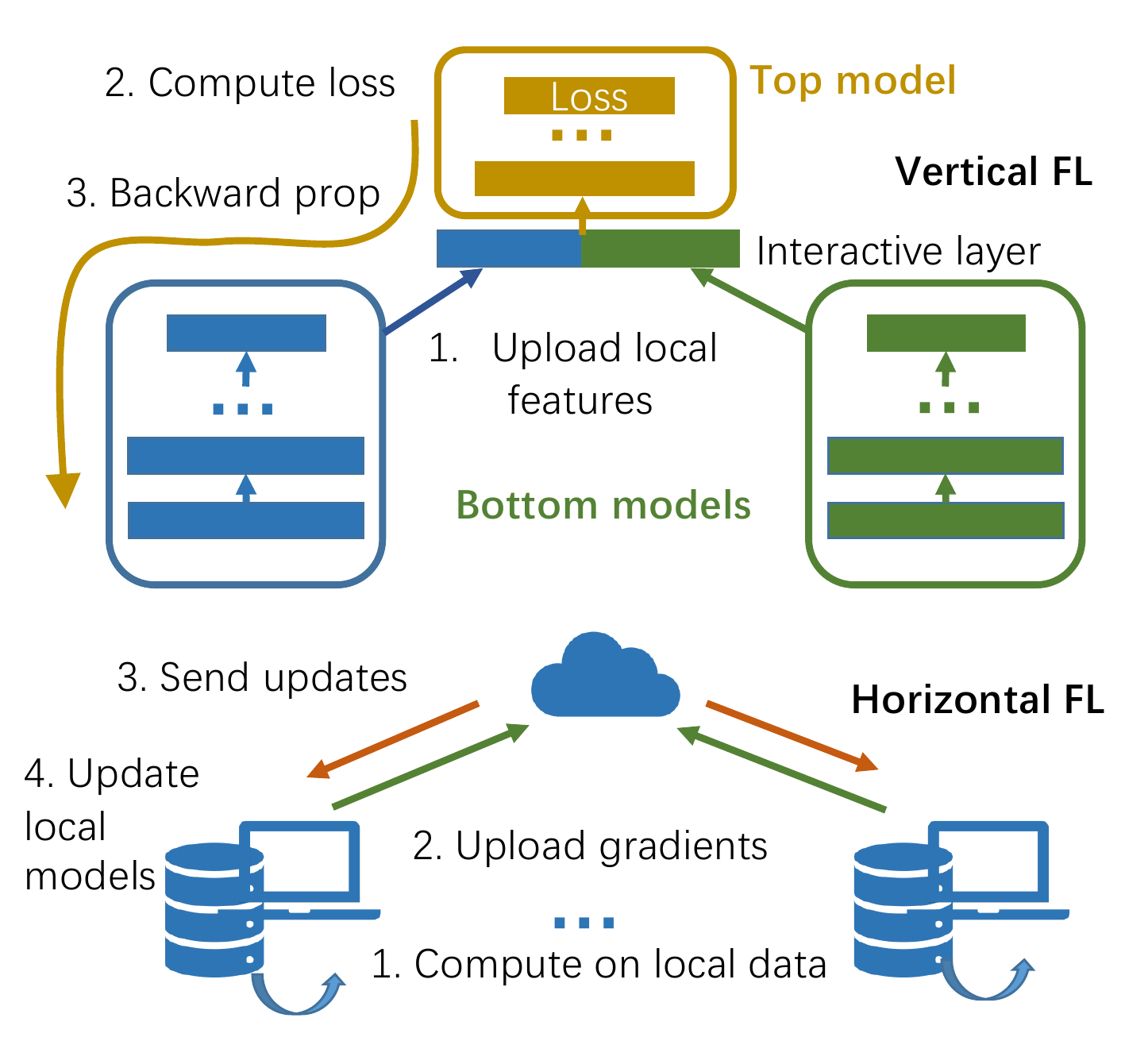}
	\caption{Federated learning frameworks.}
	\label{fig:frameworks}
\end{figure}
\vspace{-3mm}

\subsection{Vertical Federated Learning}
\label{sec:vfl}
In vertical FL, different attributes of data samples belong to different participants. Participants need to identify data samples in common by entity resolution.  The training process is provided in Fig.~\ref{fig:frameworks} (upper). Since features of the same data are distributed across multiple parties, all participants find their common data samples without compromising the non-overlapping parts of the datasets. Apart from the local datasets, each participant has its own bottom model which may be different depending on the features they process. The parties jointly build the interactive layer which puts together the intermediate features of all participants as one. The interactive layer is owned by a central server, or an arbiter, who also maintains the top model and feeds it with the output of the interactive layer. In the forward propagation, each participant feeds its input to their respective bottom models to produce the interactive-layer feature. The feature is fed to the top model to calculate the loss. The backward loop propagates the error from the output layer to the interactive layer and then to each bottom model, which gets locally updated. 

\begin{figure}[th]
	\centering
	\includegraphics[width=1\linewidth]{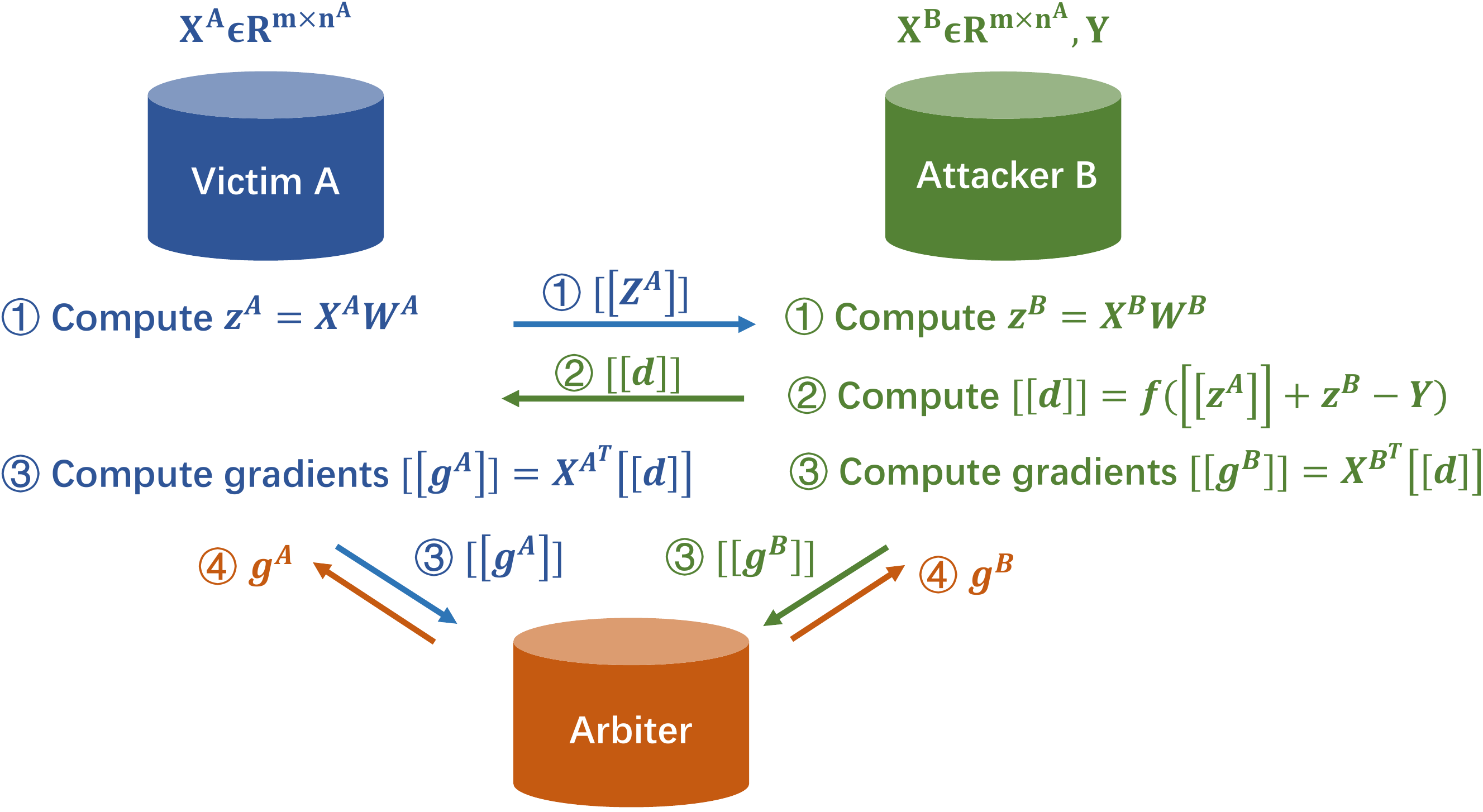}
	\caption{Procedures of vertical federated learning.}
	\label{Vertical}
\end{figure}

\textbf{FATE} uses homomorphic encryption in the training process to preserve data privacy. We use the two-party linear regression as an illustrative example. The training procedure is summarized in Fig.~\ref{Vertical}. A and B respectively own data $X^A$ and $X^B$ of which the attributes are aligned. $Y^A$ and $Y^B$ are labels of their data. $W^A, b^A$ and $W^B, b^B$ are the bottom model weights and biases. The top model minimizes the following loss function:
$$L = \frac{1}{2}(||d||^2+\alpha (||W^A||^2+||W^B||^2)),$$ where
\begin{equation}
	\label{eqn:d_eqn}
	d = f(X^{A}W^{A}) + f(X^{B}W^{B}) - Y.
\end{equation}
A and B respectively compute $z^A = f(X^{A}W^{A}) $ and $z^B = f(X^{B}W^{B})$ as the intermediate-layer features. Then A transfers the encrypted feature $[[z^A]]$ to B for B to calculate $[[d]]$. B transfers the encrypted $[[d]]$ back to A. Then A and B respectively calculate the encrypted gradients $[[g^A]]$ and $[[g^B]]$. Finally, A and B transfer the gradients to arbiter for decryption and get their gradients $g^A$ and $g^B$ respectively to update the model. For logistic regression, one only needs to replace the function $f$ with a Sigmoid function. Typically, a polynomial function is used to simulate the Sigmoid function to allow homomorphic encryption on $f$.


\subsection{Horizontal Federated Learning}
\label{sec:hfl}
In horizontal federated learning, the training data gets partitioned horizontally among parties, {\em i.e.,} data matrices or tables are partitioned by samples. Data from different participants has the same attributes. The training process of horizontal FL is shown in the Fig.~\ref{fig:frameworks} (lower). In the initial round, participants pull a randomly initialized model from the centralized server, and train the local model on their respective datasets. Each participant uploads locally computed weights/gradients to the server for aggregation. The central server maintains a global model, and uses the averaged weights/gradients to update the global model. As a final step, the server sends the updated model to each participant for the next round of local training. 


\begin{figure}[th]
	\centering
	\includegraphics[width=0.9\linewidth]{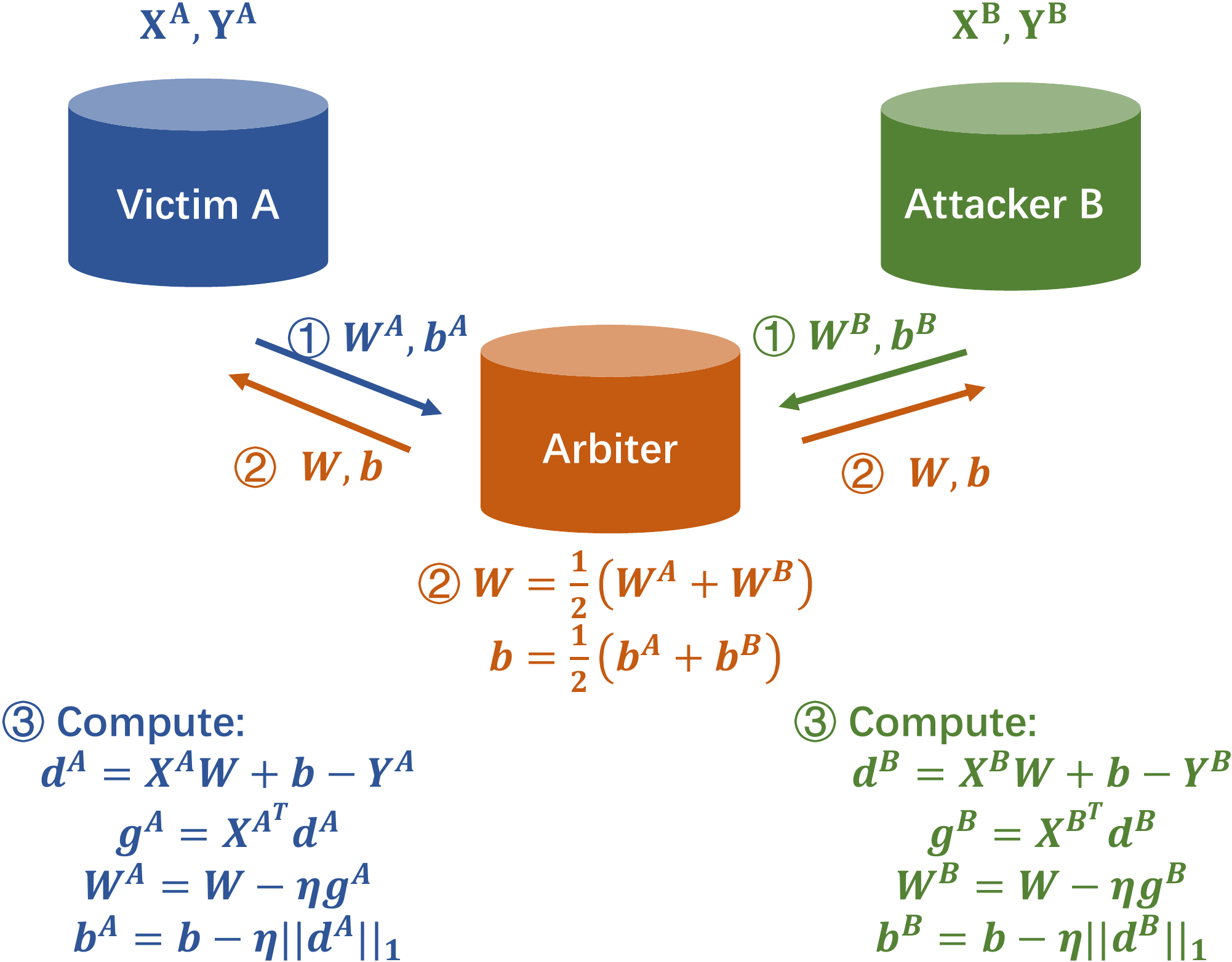}
	\caption{Procedures of horizontal federated learning.}
	\label{Horizontal}
\end{figure}

\begin{figure}[th]
	\centering
	\includegraphics[width=1\linewidth]{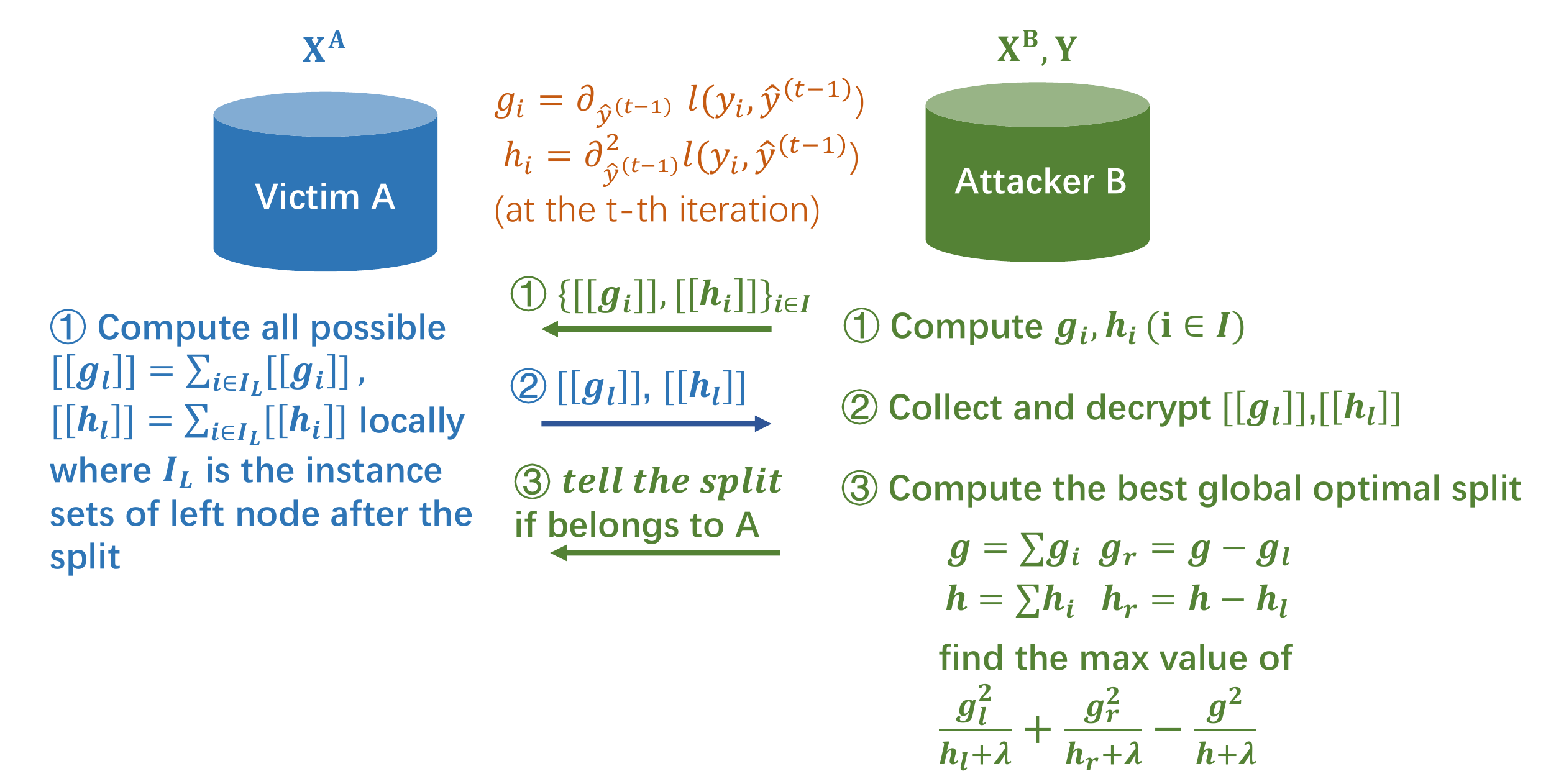}
	\caption{Procedures of SecureBoost.}
	\label{secureboost}
\end{figure}

\textbf{PySyft} is a privacy-preserving federated learning framework built on multiparty computation and differential privacy. Two-party linear regression on the horizontal FL is shown in Fig.~\ref{Horizontal}. In each training iteration, both A and B train their model locally, and upload their encrypted model parameters to arbiter for averaging. The averaged parameters are sent back to A and B for decryption.  A and B use the decrypted model parameters to update their models locally.


\subsection{SecureBoost} \label{seb}
Similar to vertical FL, SecureBoost builds a decision tree without the participating parties revealing their attributes or training instances. \textbf{SecureBoost} integrates XGBoost \cite{Chen_2016} into FATE, of which the procedure is given in Fig.~\ref{secureboost}. B iteratively uses the label and the intermediate results to calculate the best splitting point (of an attribute) for each node. In the $t$-th iteration, B calculates the first and second-order derivative $g_{i}$ and $h_{i}$ for the $i$-th item and encrypt them. The encrypted values are sent to A who sum up all items in $I_{L}$, which is the set of instances in the left node. For example, if A has three items $\{1,2,3\}$ and divide them by $\{1,2\},\{3\}$ according its attribute, A computes the encrypted values of $g_{1}+g_{2}, g_3$ and $h_{1} + h_{2}, h_{3}$. Then A sends back the encrypted sum to B. Throughout the process, A cannot obtain the values of $g_{i}$ and $h_{i}$ in plaintext but merely divide the items depending on A's features. B knows how A divides items into different sets but has no idea what feature A uses. Finally, B uses the decrypted values to determine the optimal splitting point for the next iteration.

\section{Threat Model}
We first introduce the threat model under which we discuss the privacy leakage. The attackers in our threat model are honest-but-curious, {\em i.e.,} that they follow the standard training process and the protocols but are curious about others' private data. Depending on the number of participants, the specific threat model varies a bit. 

In the {\em two-party} case, we let one of the parties, w.l.o.g., B, be the attacker and A be the victim. The arbiter is not corrupted. In the vertical FL setting, we assume the worst-case adversary B possesses labels of the training samples. In both the horizontal and vertical FL settings, B's objective is to steal A's private data. In the decision tree model, B aims to recover the range of A's private data.

In the {\em multi-party} scenario, we assume one of the participants and the arbiter are corrupted, together they try to recover all other participants' private data. This is possible in reality as the attacker might collude with the arbiter to intercept others' data. We later show that the case is equivalent to the two-party one, since all encrypted information can be decrypted at the arbiter, and can be transfered to B. 

It is worth mentioning that our proposed attacks are not straightforward under the threat model: simply deducting the attacker's knowledge from the interactive-layer features cannot exactly recover victim's data. The attacker's objective is to recover the exact values (ranges in the decision tree model) of the victims. It is our goal to demonstrate how easy it is to expose participants's data in federated learning tasks.

\section{Attacks to Regression Models}
\label{sec:formulate}
With an emphasis on the regression models, we quantify the privacy risks in both the vertical and horizontal FL settings by revealing the least amount of corrupted data that an attacker requires to reconstruct the victim's entire inputs.

\subsection{Privacy Risks in Vertical FL}
\label{sec:reg_vfl}
We use the instance of linear regression to analyze the privacy risks in vertical federated learning described in Sec.~\ref{sec:vfl}. We first look into the two-party case and then extend it to the multi-party scenario.

Inheriting the denotations in Sec.~\ref{sec:vfl}, we let $X^{A} \in \mathbb{R}^{m \times n^{A}}, X^{B} \in \mathbb{R}^{m \times n^{B}}$ where $m$ is the number of training items in common for A and B and $n^{A}, n^{B}$ are respectively the number of attributes. In the $k$-th training iteration, Attacker B decrypts its gradient as 
$$g^B_k = {X^B_k}^T(z^A_k+z^B_k-Y_k)+\alpha W^B_k,$$
where $z^A_k=X^A_k W^A$ and $z^B_k=X^B_k W^B_k$. As B knows the values of $g^B_k$, $z^B_k$, $X^B_k$, $Y_k$ and $W^B_k$, as long as ${X^B_k}^T$ has $m_k$ linearly independent rows, B can solve the linear equations to obtain $z^A$, which is originally encrypted in the training process. 

Having $m_k$ linearly independent rows means it is required that $n^B \geq m_k$, which is a necessary condition for solving $z^A$. Since the real-world data record is rarely dependent on each other, we only need $n^B \geq m_k$. In the case where $m_k \geq n^B$, we can let B fake some features which are randomly distributed to meet the requirement. Since the fake features are randomly distributed, the model would not learn anything from them and the effect of these features would reduce to $0$ as the model converges. Therefore, without hurting model accuracy, B can achieve the requirement of $n^B \geq m_k$ with fake features.

On obtaining the intermediate results $z^A$, B can invert more information w.r.t. the input given the results of two adjacent training iterations on the same batch of training data. To simplify the expression, we let $d_k = z^A_k+z^B_k-Y_k$ in the following deduction.
\begin{equation}
\label{eqn:simplify} 
\begin{aligned}
z^A_{k+1}-z^A_k& =X^A_k (W^A_{k+1}-W^A_k) \\
& =-X^A_k \eta ({X^A_k}^T d +\alpha W^A_k) \\
& =-\eta X^A_k({X^A_k}^T d) - \eta \alpha X^A_k W^A_k.
\end{aligned}
\end{equation}
By rearranging the equation, we have
\begin{equation}
\label{eqn:rearrange}
z^A_{k+1}-z^A_k(1 - \eta \alpha ) = \eta X^A_k({X^A_k}^T d).
\end{equation}
Since B obtains the values of $z^A_{k+1}, z^A_k, \eta, \alpha$ and $d$, B can rewrite Eq.~\eqref{eqn:rearrange} as $X^A {X^A}^T = C$ given $C$ a constant value. This is a quadratic equation of $X^A$ and the equation would leak the private input of A. 

Beyond the information leakage in the training phase, B can obtain more information about A in the testing phase. After the model is trained, B can steal $W^{A}$, the final model parameters of A, by sending $n^A+1$ linearly independent queries to A, and obtain equations:
$z^A_{test}=X^A_{test} W^A.$
B merely needs to act as an ordinary user and send $X^A_{test}$ to A. Since A sends unencrypted $z^A_{test}$ to B to compute the predicted value, B can acquire the value of $W^{A}$ by solving the above equation given a sufficient number of queries.

We consider the model converges in the last iteration of training, and hence the model parameters of the last training iteration $W_{last}^{A}$ can be approximated by $W^{A}$. With the approximation, we obtain the following linear equations on $X^A$:
\begin{equation}
\label{xa}
z^A_{last} \approx X^A W^A,
\end{equation}
which establishes a linear equation about $X^A$ since $z^A_{last}, W^A$ are known. To sum up, during the training and testing phases, B can obtain a quadratic equation Eq.~\eqref{eqn:rearrange} and a linear equation Eq.~\eqref{xa} on the private input of A.

The privacy risk analysis on logistic regression is almost the same except that the loss function is different. In logistic regression, Attacker B has 
$$g^B_K = {X^B_k}^T(A_k-Y_k)+\alpha W^B_k,$$ 
where $A_k = Sigmoid(z^A_k+z^B_k)$. As long as $X^B_k$ has $m_k$ linearly independent row vectors, B can solve the above linear equations to obtain $A_k$. Since Sigmoid function is approximated to polynomial function and B has the value of $z^B$, so it can get the value of $z^A$. Subsequent procedure is the same with the case of linear regression.

For the multi-party case, we assume attacker B colludes with the arbiter. Since the arbiter has the private key, it can get the gradient of the victim in plaintext, as well as the decrypted value of $d_k$ by colluding with B. Considering
$$g_k^A=X^{A^T}d_k+\alpha W_k^A,$$ 
we can take advantage of $W_{k+1}^A=W_{k}^A-\eta g_k^A$ and use two adjacent $g_k^A$s to get the linear equation of $X_k^{A^T}$:
$$g_{k+1}^A-(1-\eta \alpha )g_k^A = X^{A^T}(d_{k+1}-d_k).$$
Similarly, as long as the number of samples is larger than the number of victim's attributes, attacker can invert all the original data of victim's. Logistic regression is similar to the linear regression one with the same reason as the two party circumstance.

\subsection{Privacy Risks in Horizontal FL}
We use the linear regression framework in Sec.~\ref{sec:hfl} to analyze the privacy risks in the horizontal federated learning. We first look into the two-party case and then analyze the multi-party case. 

In each training iteration, Victim A and Attacker B obtain the average weight $W  \in R^{n\times 1}$ for the current round. $X^A \in R^{m \times n}$ is the private input of A. $m$ is the number of training samples and $n$ is the number of features. $Y^A \in R^{m\times 1}$ is the label of A's data. Obtaining the averaged weights, B can easily calculate the gradient of A:
\begin{equation}
\label{eqn:ga}
	g^A= \frac{1}{\eta}(2(W_{k}-W_{k+1}))-g^B.
\end{equation}
Since we have
\begin{equation}
\label{eqn:linear}
\begin{aligned}
g^A& = {X^A}^T (X^A W-Y^A)  \\
& = {X^A}^T (X^A W)- {X^A}^T Y^A,
\end{aligned}
\end{equation}
\begin{equation}
\label{eqn:linear1}
g^A_{k+1}-g^A_k= {X^A}^T (X^A (W_{k+1}-W_k)).
\end{equation}
Since $g^A_{k+1}, g^A_k, W_{k+1}, W_k$ are known to B, Eq.~\eqref{eqn:linear1} can be expressed as quadratic equations ${X^A}^T X^A = C,$ in a similar fashion to that of the vertical federated learning. In the multi-party case, since each participant's local weights are exposed to the arbiter, the case can be reduced to solving Eq.~\ref{eqn:linear1}.

~\\
The attacks to the regression models in the horizontal and vertical FL can be summarized as follows: 

$\diamond$ {\em We have an unknown real matrix $A \in \mathbb{R}^{m \times n}$ and a constant matrix $C  \in \mathbb{R}^{m \times m}$ such that $AA^T=C$. We also have constant vectors $W \in \mathbb{R}^{n \times 1}$ and $Z \in \mathbb{R}^{m \times 1}$ which satisfies $AW=Z$. What are the degrees of freedom of $A$ with/without the set of linear constraints?} 

\subsection{Inverting Private Inputs}
Degrees of freedom $\nu_{A}$ represents the number of variables required to be known for the attacker to invert all elements in $A$. $\nu_{A} = 0$ represents that $A$ can be entirely inverted without any additional information. $\nu_{A} = \infty$ means that the system does not leak any private information of the victim since the attacker obtains an infinite number of solutions given what it can acquire. Typically, $\nu_{A}$ is a finite number, suggesting that given at least $\nu_{A}$ data points of Victim A, Attacker B can invert the entire input of A.

To solve this problem, we first discuss the degrees of freedom of $A$ without the set of linear constraints. Since $C$ is a real symmetric matrix, we only need to care about its upper triangular matrix (row index $i \leq$ column index $j$). To be more specific, we write 
$$A =
\left(
\begin{smallmatrix}
x_{11} & \cdots & x_{1n}\\ \vdots & \ddots & \vdots\\ x_{m1} & \cdots & x_{mn}\\
\end{smallmatrix}
\right) = \left(
\begin{smallmatrix}
A_1 \\ A_2 \\ \vdots \\ A_m \\
\end{smallmatrix}
\right), C=\left(
\begin{smallmatrix}
c_{11} & \cdots & c_{1m}\\ \vdots & \ddots & \vdots\\ c_{m1} & \cdots & c_{mm}\\
\end{smallmatrix}
\right).$$
The quadratic equations can be expressed as 
$$\left\{
\begin{smallmatrix}
\sum\limits_{k=1}^{n}x_{1k}x_{1k}=c_{11} \\   \sum\limits_{k=1}^{n}x_{1k}x_{2k}=c_{12} \\ \vdots \\ \sum\limits_{k=1}^{n}x_{ik}x_{jk}=c_{ij} \\ \vdots \\ \sum\limits_{k=1}^{n}x_{mk}x_{mk}=c_{mm} \\
\end{smallmatrix}
\right. \text{or}  \left\{
	\begin{smallmatrix}
	A_1\cdot A_1=c_{11} \\
	A_1\cdot A_2=c_{12} \\ \vdots \\
	A_i\cdot A_j=c_{ij}  \\ \vdots \\
	A_m\cdot A_m=c_{mm} \\
	\end{smallmatrix}
	\right.$$
We use $\cdot$ to represent the inner product of two vectors. In the below, we will first introduce some lemmas and then our major theorems.

\begin{lemma}\label{lm1}
	The solution set to $AA^T=C$ contains at least one solution.
\end{lemma}
\begin{proof}
	Since $C$ is obtained by multiplying $A$ by $A^{T}$, we know there must exist one solution to $AA^T=C$. Since $C$ is a real-valued symmetric positive definite matrix, we can use Cholesky decomposition to find a particular solution $A^*$ such that $(A^*)(A^*)^T=C$. So the solution set to $AA^T=C$ contains at least one solution.
\end{proof}

\begin{lemma}\label{lm2}
	The degrees of freedom of $AA^T =C$ is at most $\frac{n(n-1)}{2}$.
\end{lemma}
\begin{proof}
The proof is divided into two parts. The first part will show the condition of constructing a basis of $n-1$ known row vectors of $A$. The second part will give how the rest $m-n+1$ (assuming $m\geq n-1$) row vectors can be obtained given the basis. We assume there are at least $n-1$ linearly independent row vectors in $A$, and rearrange these row vectors to form a new matrix $B$ such that:
$$B=\left(
\begin{smallmatrix}
A_1 \\ A_2 \\ \vdots \\ A_{n-1} \\
\end{smallmatrix}
\right) = \left(
\begin{smallmatrix}
x_{11} & \cdots & x_{1n}\\ \vdots & \ddots & \vdots\\ x_{n-1,1} & \cdots & x_{n-1,n}
\end{smallmatrix}
\right).$$

Assuming that $x_{11}, x_{12}... x_{1,n-1}$ in $A_1$ are given, we can obtain the value of $x_{1,n}$ by solving the equation $A_1\cdot A_1=c_{11}$. 

By assuming $n-t, \forall 1\leq t\leq n-1$ variables in $A_t$ as known variables, we can solve the $t$ equations:
\begin{equation} \label{eqn:mxn}
\left(
\begin{smallmatrix}
A_1 \\ A_2 \\ \vdots \\ A_t \\
\end{smallmatrix}
\right) \cdot A_t^T=\left(
\begin{smallmatrix}
c_{1t} \\ c_{2t} \\ \vdots \\ c_{tt} \\
\end{smallmatrix}
\right),
\end{equation}
to get $A_t$. To obtain all $n-1$ row vectors in $B$, we need to know at least $\frac{n(n-1)}{2}$ variables in the upper triangle matrix in $B$. It is worth noting that for each $t (1\leq t\leq n-1)$, if we move all known terms to one side and get the equation as $MX=N$ where $X$ are the unknowns, the coefficient matrix $M$ has to have full row rank to get all unknowns. A further discussion in Corollary~\ref{co1} would depend on this property.

Given the basis constructed by $n-1$ linearly independent vectors, we are able to solve each of the remaining $m-n+1$ unknown vectors of $A$. So far, the matrix $A$ can be completely determined. Therefore, with $\frac{n(n-1)}{2}$ known elements of $A$, one can obtain a finite number of solutions to $AA^T=C$. Since we do not use all the equality constraints of $AA^T=C$, the degrees of freedom of $A$ is at most $\frac{n(n-1)}{2}$.
\end{proof}

\begin{lemma}\label{lm3}
	The degrees of freedom of $AA^T =C$ is at least $\frac{n(n-1)}{2}.$
\end{lemma}
\begin{proof}
By Lemma~\ref{lm1}, we know there must exist one particular solution to $AA^T =C$. Let it be $A^*$. We construct the matrix $D=A^*P$, where $P$ is an $n \times n$ orthogonal matrix. Since $DD^T=(A^*P)(A^*P)^T=A^*PP^T(A^*)^T=(A^*)(A^*)^T=C$, $D$ is also a solution to $AA^T =C$. As the $n$-dimensional orthogonal matrix $P$ has $\frac{n(n-1)}{2}$ degrees of freedom, it can be obtained that $AA^T =C$ has at least $\frac{n(n-1)}{2}$ degrees of freedom.
\end{proof}

\begin{theorem}
	\label{thm1}
	Assuming $A$ contains at least $n-1$ linearly independent row vectors, the degrees of freedom of $AA^T =C$ is $\frac{n(n-1)}{2}.$
\end{theorem}
The conclusion is obvious by Lemma~\ref{lm2} and Lemma~\ref{lm3}.

From the proof of Lemma~\ref{lm2}, we can tell how to obtain each element of $A$. First, we need to construct the basis of $n-1$ known row vectors. Given the basis, we can obtain all unknowns in the remaining $m-n+1$ row vectors with no other strings attached. The necessary condition of constructing the basis is as follows: $A$ contains at least $n-1$ linearly independent row vectors, and in these $n-1$ linearly independent row vectors, at least $\frac{n(n-1)}{2}$ variables have to be known.

However, $A$ is not guaranteed to be inverted given only the necessary conditions. For example, in the case where $m=3$ and $n=3$, and $A \in \mathbb{R}^{3 \times 3},$ we have $A_1=(x_{11},x_{12},x_{13})$ and $A_2=(x_{21},x_{22},x_{23})$ being  linearly independent, and we know the value of $x_{11}$, $x_{21}$, and $x_{31}$. It is easy to find that with thosen known values, $A$ cannot be uniquely determined. Typically the couplings between different elements in the matrix is quite complicated, and hence it raises requirements to the positions of the given values if we want to restore the matrix through $\frac{n(n-1)}{2}$ known elements. In the following, we give sufficient conditions of inverting $A$:
\begin{corollary} \label{co1}
	To invert $A$ under the constraint of $AA^T =C$, the following conditions need to be satisfied:
\begin{enumerate}
	\item $A$ contains at least $n-1$ linearly independent row vectors, of which $\frac{n(n-1)}{2}$ elements are known. In the $(n-t)$-th row vector, the number of known elements should be $n-t (1 \leq t \leq n-1)$.
	\item For $1 \leq t \leq n-1,$ if we rearrange all known items in
	$$\left(
	\begin{smallmatrix}
	A_1 \\ A_2 \\ \vdots \\ A_{t-1} \\
	\end{smallmatrix}
	\right) \cdot A_t^T=\left(
	\begin{smallmatrix}
	c_{1t} \\ c_{2t} \\ \vdots \\ c_{t-1,t} \\
	\end{smallmatrix}
	\right),$$ 
	to one side to compose the new equation $MX=N$, $M$ has full row rank.
\end{enumerate}
\end{corollary}

The proof can be referred to the proof of Lemma~\ref{lm2}. 

The two conditions provide a practical way for the attacker to invert the victim's input. The 1st condition does not impose the column position requirement on the known elements within a row vector. The attacker only needs to obtain arbitrary $n-t$ known elements in the $(n-t)$-th row. In the 2nd condition, we do not impose any restriction on the order of row vectors, as long as $M$ can be inverted to form the basis.
\begin{figure}[H]
	\centering
	\includegraphics[width=0.9\linewidth]{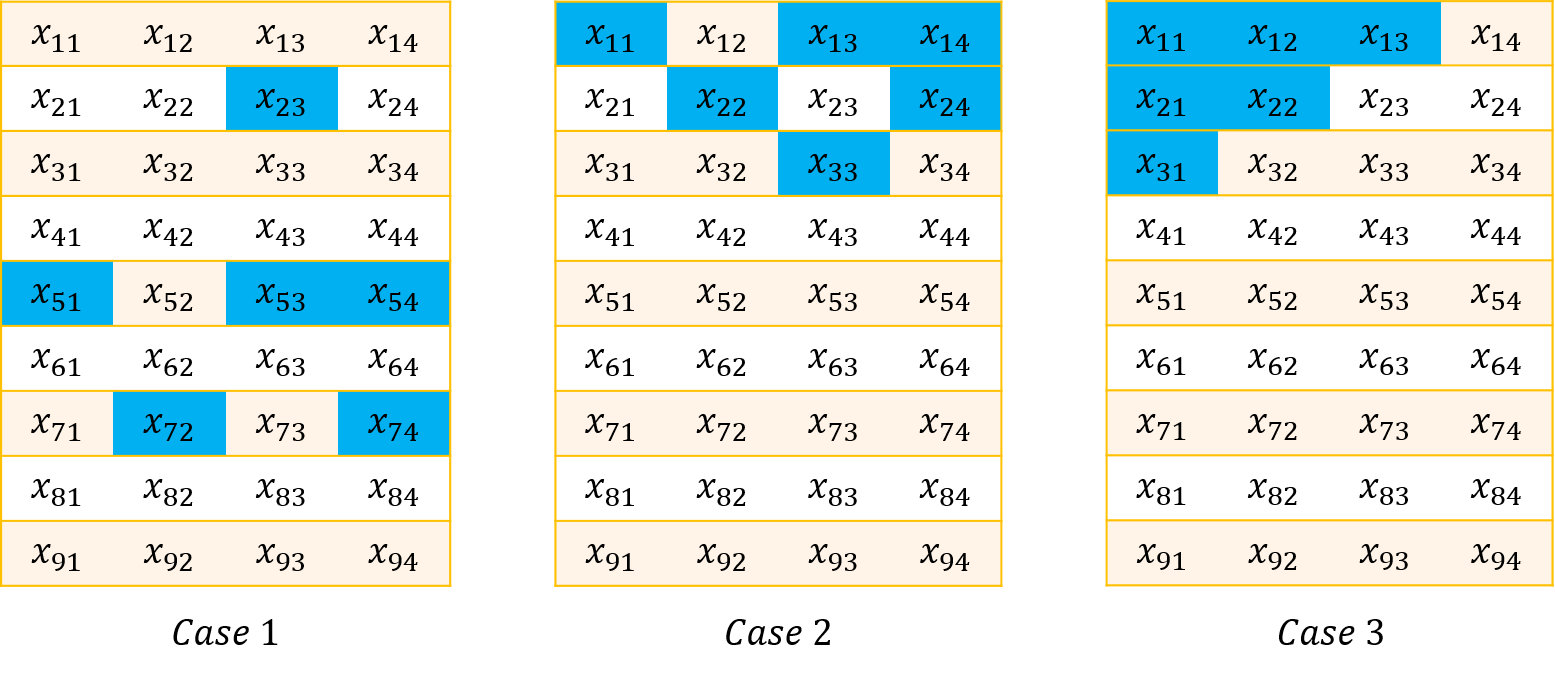}
	\caption{Three cases of known element positions.}
	\label{case}
\end{figure}
\vspace{-3mm}

For instance, $A$ is of shape $9 \times 4$ and  has degrees of freedom $6$. We assume elements in blue are known. Fig.~\ref{case} shows three cases of known positions, and they all satisfy the conditions in Corollary~\ref{co1}. $Case$ $1$ turns into $Case$ $2$ through row transformation. $Case$ $2$ and $Case$ $3$ are equivalent in solving equations. In real-world scenarios, since $A$ is a real-valued matrix and each value contains fractional digits, it is a rare condition where the vectors in the matrix are (closely) linear dependent. Therefore, the conditions in Corollary~\ref{co1} are reasonable.

\subsection{Inversion with Linear Constraints}
Now we consider the case with linear constraints $AW=Z$. The linear equations can be expressed as 
$$
\left\{
\begin{smallmatrix}
	\sum\limits_{k=1}^{n}w_{k}x_{1k}=z_{1} \\   \sum\limits_{k=1}^{n}w_{k}x_{2k}=z_{2} \\ \vdots \\ \sum\limits_{k=1}^{n}w_{k}x_{mk}=z_{m} \\
\end{smallmatrix}
\right. \text{or} \left(
\begin{smallmatrix}
	A_1 \\ A_2 \\ \vdots \\ A_m \\
\end{smallmatrix}
\right) \cdot W = \left(
\begin{smallmatrix}
	z_1 \\ z_2 \\ \vdots \\ z_m \\
\end{smallmatrix}
\right).$$
This is the case where the attacker obtains the model weights of the victim in the inference phase. Together with the quadratic equations, we want to find out the degrees of freedom of this system.

\begin{lemma} \label{lm4}
	The degrees of freedom of $AA^T =C$ and $AW=Z$ is at most $\frac{(n-1)(n-2)}{2}.$
\end{lemma}
\begin{proof}
	We assume there are at least $n-1$ linearly independent row vectors in the system. We treat $W^T$ as a set of known row vectors and replace the row vector which has $n-1$ known values in $B$ (assuming that there are $n-1$ row vectors of $B$ linearly independent):
	$B=\left(
	\begin{smallmatrix}
		W^T \\ A_2 \\ \vdots \\ A_{n-1} \\
	\end{smallmatrix}
	\right)$. Following exactly the same derivation with Lemma~\ref{lm2}, we remove $n-1$ unknows from $A$ since $W^T$ has replaced $A_1$. So the degree of freedom is reduced by $n-1$. Since Lemma~\ref{lm2} has shown that the degrees of freedom is at most $\frac{n(n-1)}{2}$, the degrees of freedom under linear constraints are at most $\frac{(n-1)(n-2)}{2}$.
\end{proof}

\begin{lemma} \label{lm5}
	The degrees of freedom of $AA^T =C$ and $AW=Z$ is at least $\frac{(n-1)(n-2)}{2}.$
\end{lemma}
\begin{proof}
	We construct matrix $F=(A^T, W)^T$, where $F \in \mathbb{R}^{(m+1) \times n}$. Let 
	$$G=FF^T=
	\left(
	\begin{smallmatrix}
		AA^T & AW \\ W^TA^T & W^TW
	\end{smallmatrix}
	\right)=
	\left(
	\begin{smallmatrix}
		C & Z \\ Z^T & W^TW
	\end{smallmatrix}
	\right),
	$$ where $G \in \mathbb{R}^{(m+1) \times (m+1)}$. Since the values of $C$, $Z$ and $W^TW$ are known, we know all elements in $G$. By Thm.~\ref{thm1}, the degrees of freedom of $F$ is $\frac{n(n-1)}{2}$ under the constraint $FF^T=G$. By constructing a basis containing $W^T$ as a row vector, the degrees of freedom of $F$ reduces by at most $n-1$, since there are $n$ known values in $W^T$. Hence the degrees of freedom of $F$ is at least $\frac{(n-1)(n-2)}{2}$.
\end{proof}

\begin{theorem} \label{thm2}
	Assuming $(A^T, W)^T$ has at least $n-1$ linearly independent row vectors, the degrees of freedom of $AA^T =C$ and $AW=Z$ is $\frac{(n-1)(n-2)}{2}.$
\end{theorem}
The conclusion can be easily obtained by Lemma~\ref{lm4} and Lemma~\ref{lm5}.

Similarly, we can deduct the sufficient conditions of inverting $A$ from both the quadratic and linear constraints. We do not repeat it here.

So far we have answered the question proposed at the beginning of this section. The attacks to the regression models in the horizontal and vertical FL can be summarized as follows: $m$ represents the number of training samples, and $n$ denotes the number of attributes owned by the victim. In vertical FL, the attacker has an unknown real matrix $A \in \mathbb{R}^{m \times n}$ and a constant matrix $C  \in \mathbb{R}^{m \times m}$ such that $AA^T=C$. Typically, $m$ is much larger than $n$, and thus it is reasonable for the attacker to acquire $O(n^{2})$ of data points to invert $O(mn)$ data points of the victim. In horizontal FL, the attacker has an unknown real matrix $A \in \mathbb{R}^{m \times n}$ and a constant matrix $C  \in \mathbb{R}^{n \times n}$ such that $A^TA=C$. The case is different from vertical FL in that the attacker has to acquire $O(m^{2})$ data points to invert $O(mn)$ data points of the victim. Since we usually have more data records than attributes, such a requirement is not easy to fulfill.

\section{Attacks to Decision Tree} \label{sec:tree}
In this section, we first analyze the privacy risks faced by SecureBoost, following the setting in Sec.~\ref{seb}. We found that even with encrypted data exchange, private input of the victim can still be leaked during training. Then we introduce our designed attack.

\subsection{Privacy Risks}
During the training of SecureBoost, the active party A with labels leads the training process, other passive parties without labels provide possible dividing methods according to their own attributes. After generating the final model, party A acquires the structure of subtrees, including leaf nodes and each internal node to which party it belongs. We regard A as the attacker and passive parties as the victims and aim to find out the privacy risk of each victim.

At the inference phase, the attacker repeatedly submit queries to obtain the attribute of each node. We particularly design attacks in the federated learning, as the attacker is able to find out the path that each training instance traverses from root to leaf. With the node attributes recovered in the inference phase, the attacker is able to tell the ranges of the victim's training input according to its traversed path on the tree.

\begin{figure}[H]
	\centering
	\includegraphics[width=0.9\linewidth]{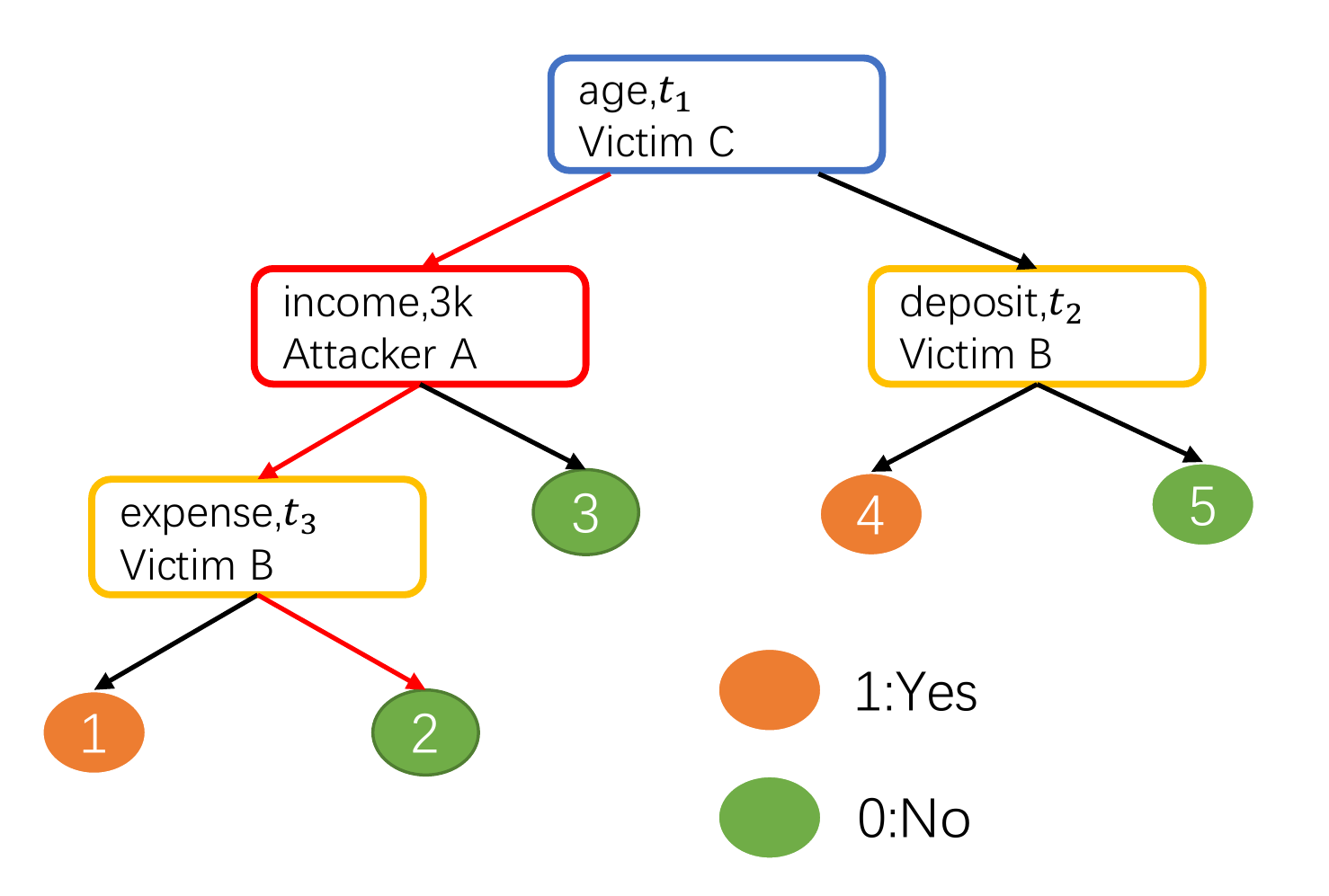}
	\caption{SecureBoost Attack}
	\label{steal}
\end{figure}
\vspace{-3mm}

\subsection{Attack Design}
A critical step of our attack is to find out the threshold of the attribute represented by each node at the inference phase. To achieve this, the attacker sends an arbitrary query initially to the root of a tree at the inference phase, and selects the path in a top-down manner. The victim participates in this process and decides the path that the query instance has to travel at its nodes. Observing how the victim deals with the query, the attacker sends a second one with a smaller range trying to approach the dividing point of the victim's nodes. This process iteratively carries on with a narrowing range until reaching the required accuracy.

We illustrate the procedure by an example in Fig.~\ref{steal}. The red node belongs to the attacker and the yellow ones and the blue one belong to the victims. The green and orange nodes are all leaves (outputs). The attacker aims to obtain values of $t_1$, $t_2$, $t_3$. Initially, the attacker denotes the upper bound and lower bound of the input feature as UB and LB. For example, for the age feature, we let UB=100, LB=0. Then the attacker generates $m$ input queries uniformly distributed across the interval. Letting $m=11$, the attacker generates a sequence in $[0,100]$ with a common difference $\frac{100-0}{m-1}=10$. The attacker feeds the $11$ input instances to the decision tree and identifies the input leading to the jump from $\{1,2,3\}$ to $\{4,5\}$ in the output. Assuming such a jump occurs at the input interval $[10,20]$, the attacker further divides such an interval to create new input instances until the range value is less than the precision threshold $\epsilon$. The arithmetic mean of the final interval can be considered as the dividing point of the node. The entire procedure of the SecureBoost attack can be found in Alg.~\ref{alg:secureboostatk}.

Given the precision threshold, UB and LB, the attacker needs to send $n_{q}$ queries to be able to obtain the input range of the victim:
$$n_{q}=\textrm{ceil}( \log_{m-1}{\frac{UB-LB}{\epsilon}} ).$$
Let the number of victim nodes be $n_{v}$, then the total number of queries required to attack a secureboost model is 
$$N_{Q}=n_{v}\times n_{q}.$$

\begin{algorithm}[H] 
	\renewcommand{\algorithmicrequire}{\textbf{Input:}}
	\renewcommand{\algorithmicensure}{\textbf{Output:}}
	\caption{\textbf{SecureBoost Attack}} 
	\label{alg:secureboostatk}
	\begin{algorithmic}[1] 
		\REQUIRE Information of all trees: $tree\_info$, including node ID, whether it is an attacker node, whether the dividing point is known, feature ID of the node, the weights of its left child and its right child; the path of each tree from the root node to the leaf node: $tree\_paths$; random input queries generated: $q$;
		\ENSURE A stack of nodes that satisfy the threshold condition: $unknown\_nodes$
		\FOR{$tree$ in $tree\_info$}
		\FOR{$path$ in $tree\_paths$[$tree$]}
		\STATE $unknown\_nodes$ $\leftarrow \{\}$;
		\FOR{$node$ in $path$}
		\IF{$node$ is an attacker node \OR dividing point of $node$ is known}
		\STATE Change $q$ to a value so that $q$ can go down the path;
		\ELSE
		\STATE $unknown\_nodes$.push($node$);
		\ENDIF
		\ENDFOR
		\WHILE{not $unknown\_nodes$.empty()}
		\STATE $node \leftarrow$ $unknown\_nodes$.pop;
		\WHILE{$UB - LB > \epsilon$}
		\STATE query array = generate\_query (query $q$, feature ID, LB, UB); \COMMENT{divide the interval of the targeted feature}
		\STATE attacker features, victim features = split\_features (query array);
		\STATE results = make\_prediction (attacker features, victim features) \COMMENT{get leaf ids}
		\STATE LB, UB= find\_split (query array, results) \COMMENT{find the dividing points from results}
		\ENDWHILE
		\STATE Update $tree\_info$;
		\ENDWHILE
		\ENDFOR
		\ENDFOR
	\end{algorithmic} 
\end{algorithm}

\section{Experiment}
We conduct a set of experiments to see how our attacks perform in real-world federated learning systems. The experimental results are provided in this section.

\begin{table}[]
	\caption{Description of Datasets}
	\label{table:datasets}
	\scalebox{0.66}{
		\begin{tabular}{|l|l|l|l|}
			\hline
			Dataset                                                         & Description                                                                                                                                                                            & Attributes & Samples \\ \hline
			Breast cancer                                                    & Follow-up data for breast cancer cases.                                                                                                                                             & 30         & 198     \\ \hline
			Motor-tempreture                                                & Motor sensor data.                                                                                                              & 11         & 998070  \\ \hline
			Vehicle scale                                                   & \begin{tabular}[c]{@{}l@{}}3D objects within a 2D image by application of an shape\\ feature extractors to the 2D silhouettes of the objects.\end{tabular}                & 18         & 946     \\ \hline
			Iris                                                            & Iris species with 50 samples each and some properties.                                                                                                                                 & 4          & 150     \\ \hline
			\begin{tabular}[c]{@{}l@{}}Residential\\ -building\end{tabular} & \begin{tabular}[c]{@{}l@{}}Construction cost, sale prices, and so on corresponding \\to real estate single-family residential apartments\end{tabular} & 107        & 372     \\ \hline
			Fish                                                            & A record of fishes in fish market sales.                                                                                                                               & 5          & 35      \\ \hline
			Red-wine                                                        & Red variant of the Portuguese "Vinho Verde" wine.                                                                                                                                      & 11         & 2449    \\ \hline
			House                                                           & Each record describes a Boston suburb or town.                                                                                                                         & 6          & 505     \\ \hline
			
		\end{tabular}
	}
\end{table}

\begin{figure}[th]
	\centering
	\includegraphics[width=0.7\linewidth]{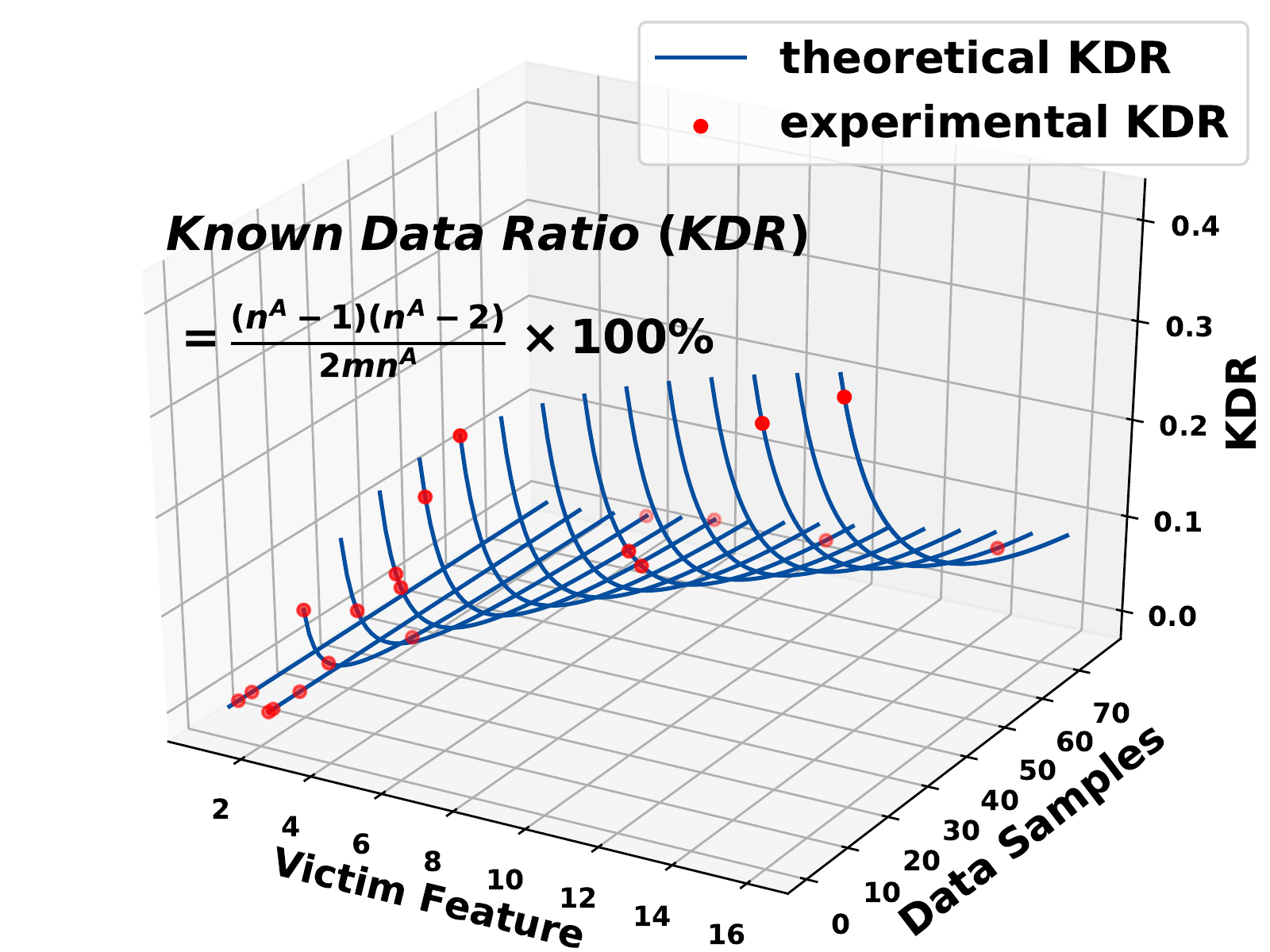}
	\caption{Known data ratio (in vertical FL).}
	\label{fig:KDR}
\end{figure}

\subsection{Implementation}
\textbf{Datasets.}
We summarize the datasets  used in experiments in Table~\ref{table:datasets}. As for attacks to linear regression models in VFL, we adopt Motor-temperature, Red wine and Residential building. For attacks to logistic regression in VFL , we use Breast cancer, Iris and Residential Building. Attacks to linear regression models in HFL are launched to datasets such as Fish, Motor-temperature, Red wine and House. Finally, Secureboost attacks are performed on Breast cancer, Iris and Vehicle scale datasets.

\textbf{Frameworks.}
{\tt FATE} provides a concrete implementation of linear regression, logistic regression and SecureBoost. We implement attacks on the regression models under vertical FL on {\tt FATE}. Since the attack to the multi-party case is similar to the two-party case, we focus on the latter. In the two-party VFL, the attacker calculates the intermediate variable $d$ and the feature $z^A$ in each training iteration according to the description in Sec.~\ref{sec:reg_vfl}. At the testing phase, the attacker submits a set of linearly independent queries to retrieve the values of $z^A_{test}$ without any encryption and uses $z^A_{test}=X^A_{test} W^A$ to restore the model weights $W^A$. Particular to the multi-party case, the arbiter collects $g^A$ and the attacker obtains the plaintext $d$ by colluding with the arbiter.

The attack to linear regression in horizontal FL is implemented under the framework of {\tt PySyft}. The attacker records the averaged model weights $W_k$ and its own weights $W^B_k$ to calculate the victim's weights $W^A_k$. The victim's gradients can be obtained by Eq.~\eqref{eqn:ga}. By feeding both the quadratic and/or linear equations to a solver, the attacker inverts the private input of the victim. The solver runs on AMD Ryzen5 3600 CPU with $16$GB memory.

We implement the attack to SecureBoost on {\tt FATE} framework by logging relevant information. The attack runs on Intel Xeon CPU E5-2630 v3 with $251$GB memory. Since the running time of reading and writing the log file is much longer than the actual training process, we need to reduce the number of I/Os to improve attack efficiency. Hence we group a number of inputs in batches to feed to the tree and perform block search. By this way, we improve the efficiency by $\log_{2}{n}$ where $n$ is the number of inputs in a batch.

\begin{table}[H]
	\centering
	\caption{Attack Linear Regression Models in Vertical FL}
	\label{table:vertical linear regression}
	\scalebox{0.72}{
		\begin{tabular}{|l|l|l|l|l|l|l|l|}
			\hline
			Dataset                                                                         & \begin{tabular}[c]{@{}l@{}}Attacker \\ feature\end{tabular} & \begin{tabular}[c]{@{}l@{}}Victim \\ feature\end{tabular} & \begin{tabular}[c]{@{}l@{}}Fake \\ feature\end{tabular} &Samples & Time(s) & \begin{tabular}[c]{@{}l@{}}Relative \\ Error\end{tabular} & KDR\\ \hline
			\multirow{4}{*}{\begin{tabular}[c]{@{}l@{}}Motor-\\ temperature\end{tabular}}    & $9$                                                           & $2$                                                         & $0$   &$9$                                                    & $0.015$   & $4.3\times 10^{-4}$ & $0$                                              \\ \cline{2-8} 
			& $7$                                                           & $4$                                                         & $0$     &$7$                                                  & $0.14$    & $8.2\times 10^{-5}$ & $10.7\% $                                              \\ \cline{2-8} 
			& $6$                                                           & $5$                                                         & $2$       &$8$                                                & $0.312$   & $2.3\times 10^{-3}$ & $15\% $                                                \\ \cline{2-8} 
			& $4$                                                           & $7$                                                         & $3$        &$7$                                               & $0.781$   & $2.2\times 10^{-4}$ & $30.6\% $                                                \\ \hline
			\multirow{3}{*}{\begin{tabular}[c]{@{}l@{}}Red-\\ wine\end{tabular}}             & $8$                                                           & $3$                                                         & $0$  &$8$                                                    & $0.063$   & $3.6\times 10^{-6}$ & $4.1\% $                                              \\ \cline{2-8} 
			& $6$                                                           & $5$                                                         & $3$        &$9$                                               & $0.282$   & $4.2\times 10^{-4}$ & $13.3\% $                                                \\ \cline{2-8} 
			& $5$                                                           & $6$                                                         & $2$        &$7$                                               & $0.532$   & $3.1\times 10^{-3}$ & $23.8\% $                                                \\ \hline
			\multirow{3}{*}{\begin{tabular}[c]{@{}l@{}}Residential-\\ building\end{tabular}} & $101$                                                         & $6$                                                         & $3$    &$76$                                                   & $0.704$   & $1.9\times 10^{-4}$ & $2.2\% $                                                \\ \cline{2-8} 
			& $72$                                                          & $35$                                                        & $0$        &$67$                                               & $593.812$ & $1.8\times 10^{-6}$ & $23.9\% $                                              \\ \cline{2-8} 
			& $55$                                                          & $52$                                                        & $0$       &$55$                                                & $2541.02$ & $8.3\times 10^{-5}$ &$44.6\% $                                              \\ \hline
		\end{tabular}
	}
\end{table}
\begin{table}[H]
	\centering
	\caption{Attack Logistic Regression Models in Vertical FL}
	\label{table:vertical logistic regression}
	\scalebox{0.72}{
		\begin{tabular}{|l|l|l|l|l|l|l|l|}
			\hline
			Dataset                                                                         & \begin{tabular}[c]{@{}l@{}}Attacker \\ feature\end{tabular} & \begin{tabular}[c]{@{}l@{}}Victim \\ feature\end{tabular} & \begin{tabular}[c]{@{}l@{}}Fake \\ feature\end{tabular} & Samples & Time(s) & \begin{tabular}[c]{@{}l@{}}Relative \\ Error\end{tabular} & KDR\\ \hline
			\multirow{5}{*}{Breast}                                                          & $27$                                                          & $3$                                                         & $0$                                                       &$27$& $0.062$   & $2.4\times 10^{-4}$ & $1.2\% $                            \\ \cline{2-8} 
			& $20$                                                          & $10$                                                        & $0$ &$20$                                                      & $3.187$   & $8.4\times 10^{-5}$ & $18\% $                            \\ \cline{2-8} 
			& $20$                                                          & $10$                                                        & $3$  &$23$                                                     & $3.344$   & $2.9\times 10^{-4}$ & $15.7\% $                            \\ \cline{2-8} 
			& $16$                                                          & $14$                                                        & $0$  &$16$                                                     & $10.891$  & $1.6\times 10^{-3}$ & $34.8\% $                            \\ \cline{2-8} 
			& $14$                                                          & $16$                                                        & $3$   &$17$                                                    & $17.125$  & $8.4\times 10^{-4}$ & $38.6\% $                            \\ \hline
			\multirow{5}{*}{Iris}                                                            & $3$                                                           & $1$                                                         & $3$   &$6$                                                    & $0.015$   & $1.6\times 10^{-12}$ & $0$                            \\ \cline{2-8} 
			& $3$                                                           & $1$                                                         & $0$     &$3$                                                  & $0.015$   & $4.7\times 10^{-14}$ & $0$                             \\ \cline{2-8} 
			& $2$                                                           & $2$                                                         & $0$      &$2$                                                 & $0.016$   & $2\times 10^{-4}$ & $0$                              \\ \cline{2-8} 
			& $2$                                                           & $2$                                                         & $1$       &$3$                                                & $0.016$   & $5.2\times 10^{-5}$ & $0$                            \\ \cline{2-8} 
			& $1$                                                           & $3$                                                         & $2$    &$3$                                                   & $0.047$   & $3.5\times 10^{-5}$ & $11.1\% $                              \\ \hline
			\multirow{3}{*}{\begin{tabular}[c]{@{}l@{}}Residential-\\ building\end{tabular}} & $103$                                                         & $4$                                                         & $1$    &$76$                                                   & $0.219$   & $2.7\times 10^{-4}$ & $1.0\% $                            \\ \cline{2-8} 
			& $97$                                                          & $10$                                                        & $0$   &$69$
			& $3.922$   & $5.7\times 10^{-4}$ & $5.2\% $                             \\ \cline{2-8} 
			& $92$                                                          & $15$                                                        & $0$      &$67$                                                 & $16.281$  & $1.9\times 10^{-4}$ & $9.1\% $                              \\ \hline
		\end{tabular}
	}
\end{table}

\subsection{Attacks to Regression Models}
\textbf{Two-Party Vertical FL.} We launch attacks to linear and logistic regression models on three datasets and report the results in Table~\ref{table:vertical linear regression} and Table~\ref{table:vertical logistic regression} respectively. We let the attacker and the victim own different proportions of data features, and see whether the victim's data can be entirely inverted given a proportion of known data. The running time of the attack is also given, which grows with the number of features. The relative reconstruction error (mean error/input mean) is also provided to show the inversion accuracy.

Since linear equations are all incorporated in the attack, we define the {\em known data ratio} (KDR) as the ratio between the amount of known data for the attacker to invert the victim's inputs and the amount of data owned by the victim:
$$KDR_{v} = \frac{(n^{A} - 1)(n^{A} - 2)}{2mn^{A}},$$
where $n^{A}$ is the number of victim's attributes and $n^A-2\leq m \leq n^B$ is the number of data records. Smaller KDR value means that the attacker gains more information about the victim and thus the more severe the privacy leakage is. For a better view of the proportion of known data versus all the training data inverted, we visualize the relation between KDR and the number of victim features ($n^A$) as well as data samples ($m$) in Fig.~\ref{fig:KDR}. As shown by the figure, the less the victim's features, or the higher the number of data samples, the lower the KDR value, indicating that it is threatening to invert all victim's inputs when it has many data samples but few attributes. Nevertheless, throughout all cases, KDR remains well under $50\%$. The results also support that the attacker can enhance the attack by adding fake features. 

According to Thm.~\ref{thm2}, when the number of victim's attributes is equal to or less than $2$, the degree of freedom of the equation system is $0$, and hence the attacker could practically invert all data without requiring further information. In these cases, KDR $=0$ held with the experimental results in Table.~\ref{table:vertical linear regression} and Table.~\ref{table:vertical logistic regression}.  For the cases where the number of victim's attributes is greater than $2$, we need to additionally know $\frac{(n^A-1)(n^A-2)}{2}$ data points, and KDR $> 0$ in these cases. Typically, the more the number of victim's features, the higher the KDR. In the case where $m \geq n^B$, we let the attacker fake a number of features to be able to invert victim's inputs.

\begin{figure}[th]
	\centering
	\includegraphics[width=0.7\linewidth]{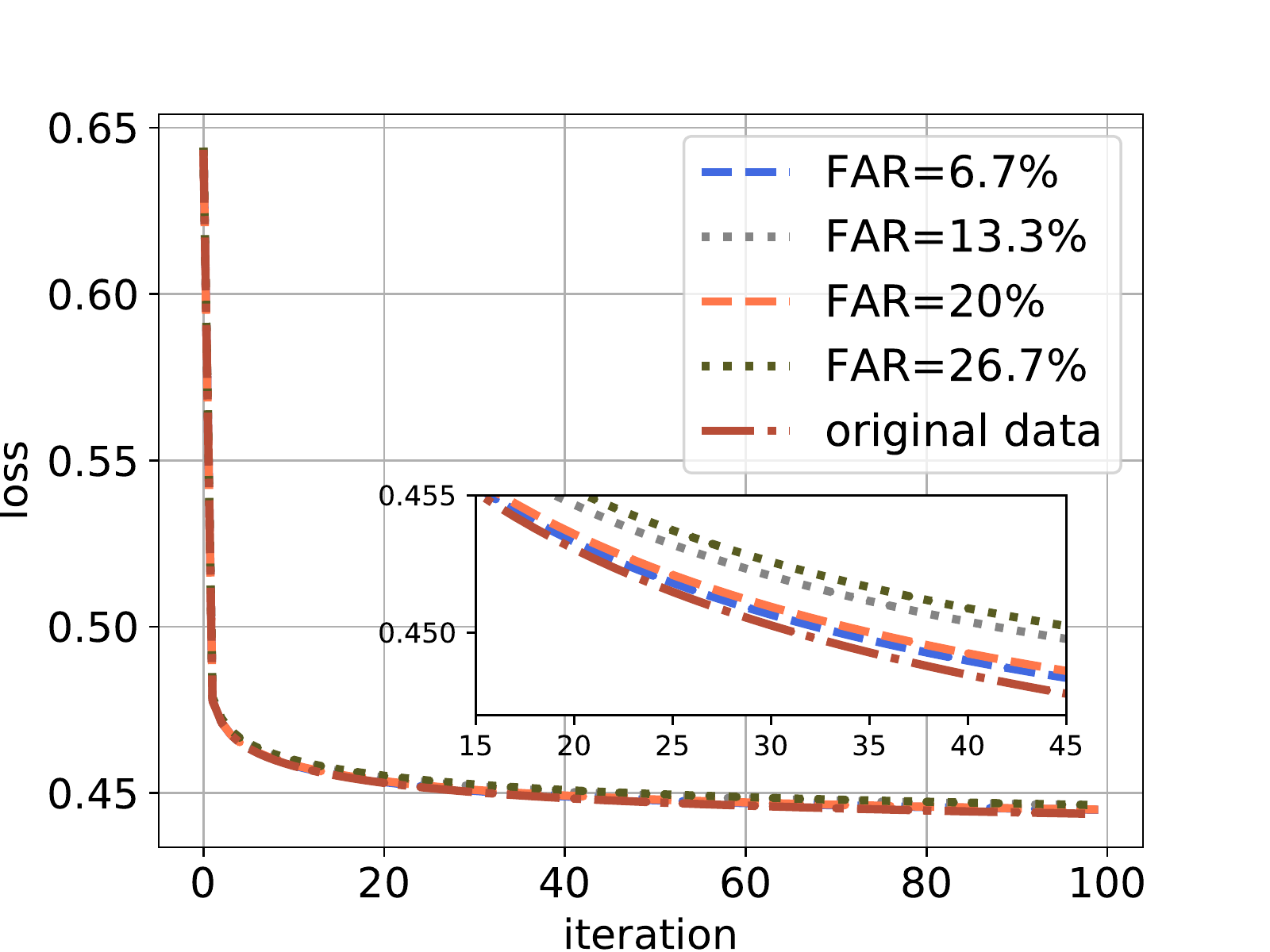}
	\caption{Breast Cancer dataset: model losses v.s. different fake attribute rates throughout training iterations. Different FARs make almost no difference to the vaildity of the model.}
	\label{fig:model_validity}
\end{figure}

To invert victim's data, we let the attacker fake attributes with relatively small randomly distributed values. To verify the assumption that adding small fake values does not hurt the validity of the model, we conduct experiments to explore the influence of such fake attributes. For datasets with total number of attributes $n_o$, we set the number of fake attributes to be $n_f$, and compare losses across training iterations under different {\em fake attribute rate} (FAR) where FAR $={n_f}/{n_o}$. We show the results on Breast Cancer dataset as an example in Fig.~\ref{fig:model_validity}. It is intuitive that the smaller the FAR is, the closer the training model behavior is to the original one. Actually, despite different FARs, the model behaves almost the same with the original one. Hence we conclude that it is practical for the attacker to inject fake attributes without being detected. Furthermore, as linearly independent vectors are required to solve equations in our attack, we find it helpful to add tiny error to the attacker's attributes. Floating-point error caused by Python calculation is sufficient to serve as fake attributes.

\textbf{Two-Party Horizontal FL.} We also launch attacks to the linear regression model under two-party HFL. Results are shown in Table~\ref{table:horizontal}. It is notable that the information the attacker can acquire in horizontal setting is much less than the vertical one. On one hand, the attacker cannot get the linear constraint as in vertical FL. On the other, the attributes of the victim and the attacker are the same, and the KDR is calculated as 
$$KDR_{h} = \frac{m(m - 1)}{2mn} = \frac{(m-1)}{2n},$$
where $m$ is the number of data records owned by the victim and $n$ is the number of attributes given $m \leq n$. If $m = n$, KDR can be very high, even close to $50 \%$ to invert the victim's inputs as we can see in Table~\ref{table:horizontal}. It is also reflected that the KDR tends to be small when the victim samples get smaller or the attribute number gets larger. Different from vertical FL, we are restricted from injecting more attributes to decrease KDR. When $m >n$,  $KDR_h = \frac{(m-1+m-n) \times n}{2mn} = \frac{2m-n-1}{2m} \geq 50 \%$ according to the proof of Lemma~\ref{lm2}. As a result, the cost of restoring victim data is very high that the attacker is required to know some attributes of every data record.

\begin{table}[H]
	\centering
	\caption{Attack Linear Regression Models in Horizontal FL}
	\label{table:horizontal}
	\scalebox{0.80}{
	\begin{tabular}{|l|l|l|l|l|l|l|}
	\hline
	Data Set                                                     & feature& \begin{tabular}[c]{@{}l@{}}Victim \\ samples\end{tabular}&\begin{tabular}[c]{@{}l@{}}Attacker \\ samples\end{tabular} & Time(s) & \begin{tabular}[c]{@{}l@{}}Relative \\ Error\end{tabular} & KDR                \\ \hline
	Fish  & $5$ & $5$ &$27$ & $0.219$ & $2.8\times10^{-10}$ & $40\% $                \\ \hline
	\begin{tabular}[c]{@{}l@{}}Motor-\\ temperature\end{tabular} & $11$  & $11$ &$60$ & $1.548$  & $1.4\times 10^{-9}$ & $45.4\% $
	\\ \hline
	Red-wine &$11$ &$4$ &$120$ &$0.516$ & $5.2\times 10^{-12}$ & $13.6\% $
	\\ \hline
	House &$6$ &$3$ &$24$ &$0.171$  & $1.0\times 10^{-12}$ & $16.6\% $
	\\ \hline
\end{tabular}
}
\end{table}
\begin{figure}[th]
	\centering
	\includegraphics[width=0.7\linewidth]{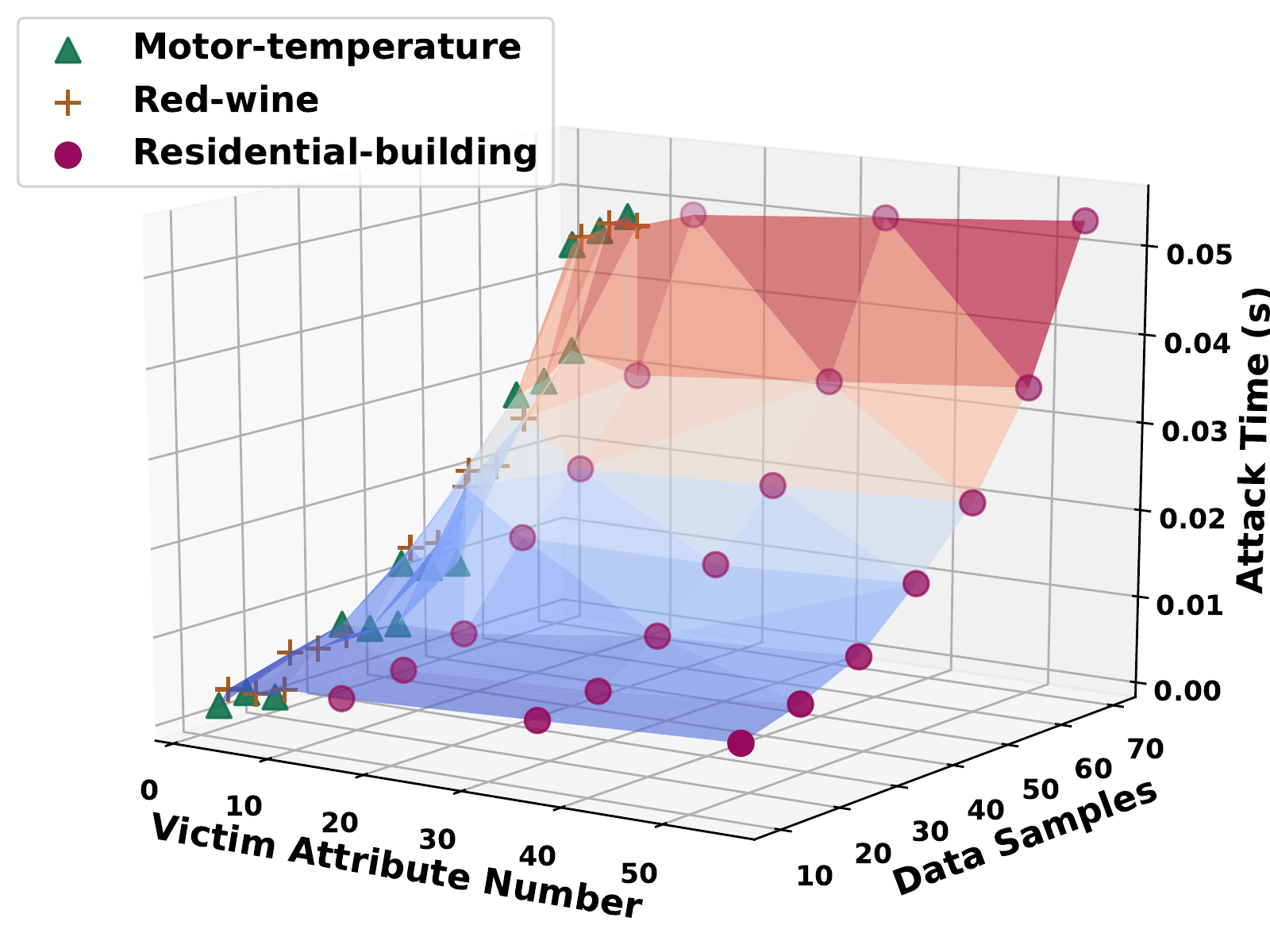}
	\caption{Running time of attacks to linear regression model in multi-party VFL.}
	\label{fig:multiple_vertical_linear_regression}
\end{figure}
\begin{figure}[th]
	\centering
	\includegraphics[width=0.7\linewidth]{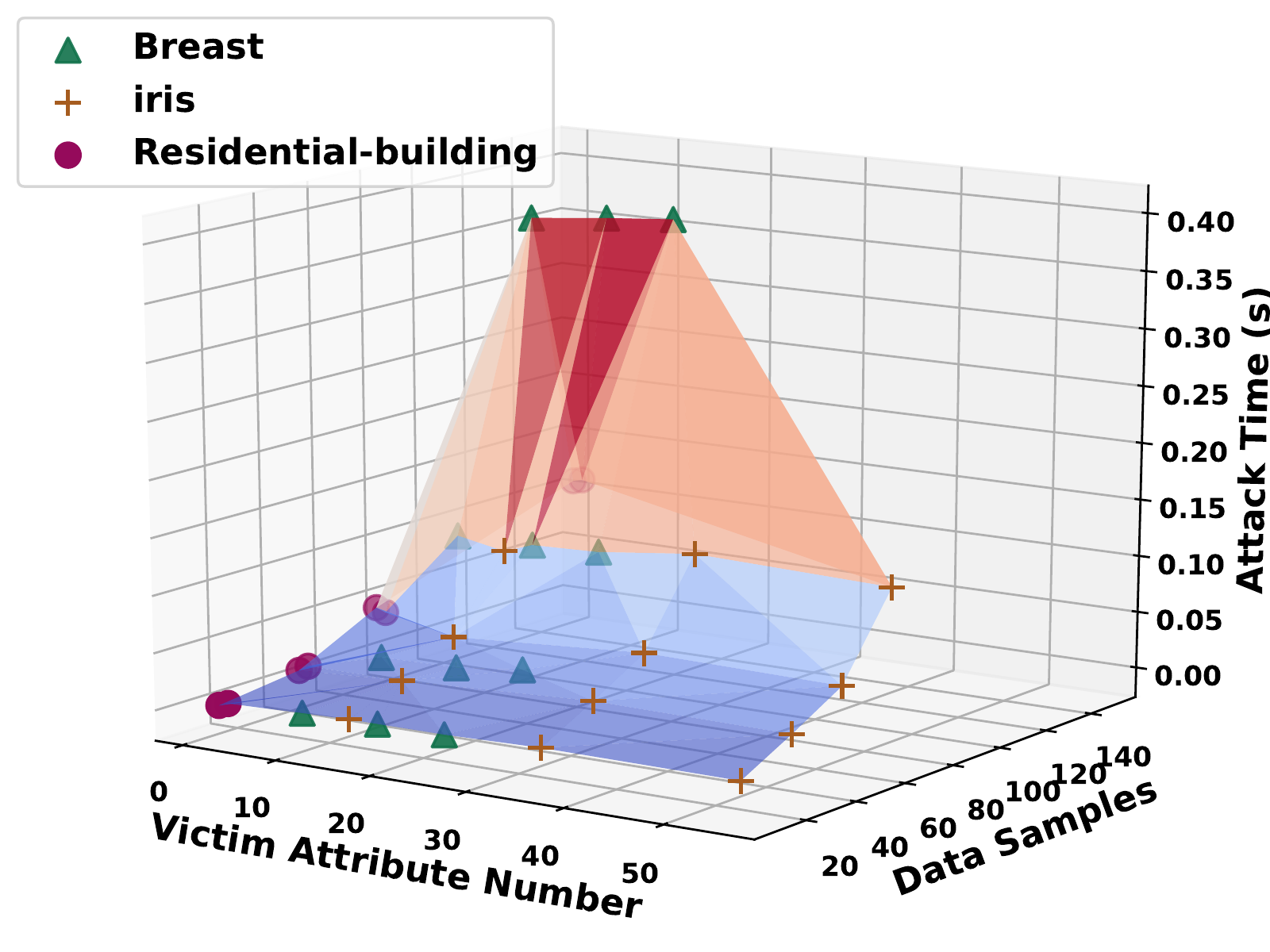}
	\caption{Running time of attacks to logistic regression model in multi-party VFL.}
	\label{fig:multiple_vertical_logistic_regression}
\end{figure}

\textbf{Multi-Party Vertical FL.} In multi-party FL settings, the attacker is able to collude with the arbiter. We perform attacks to the linear and logistic regression models on three datasets. The running time of each attack to datasets of different victim attribute dimensions and data samples are given in Fig.~\ref{fig:multiple_vertical_linear_regression} and Fig.~\ref{fig:multiple_vertical_logistic_regression} respectively. At each spot, the attacker and victims have different proportions of attributes, but the victim's data can be entirely inverted in all cases. In our experiments, the order of magnitude of the relative reconstruction error is all under $10^{-10}$, which is mainly caused by the floating-point calculation error.

The running time of the attack grows with the number of data samples $m$ but does not vary much with the number of attributes $n$. It agrees with the observation that one needs to obtain $m$ linearly independent $d$'s (Eqn.~\ref{eqn:d_eqn}) to solve equations across different iterations.  As a result, this part of running time is only related to $m$.

\begin{table}[H]
	\centering
	\caption{SecureBoost Attack on Different Datasets}
	\label{table:secureboost}
	\scalebox{0.65}{
	\begin{tabular}{|c|c|c|c|c|c|c|c|c|c|}
		\hline
		Data Set& \begin{tabular}[c]{@{}l@{}}Victim\\ feature\end{tabular} 
		&  \begin{tabular}[c]{@{}l@{}}Attacker \\feature\end{tabular}
		& \begin{tabular}[c]{@{}l@{}}feature\\range\end{tabular}
		&Subtrees&\begin{tabular}[c]{@{}l@{}}Victim\\ Nodes\end{tabular}&Queries&$\epsilon$
		&\begin{tabular}[c]{@{}l@{}}Records\\/Query\end{tabular}
		&\begin{tabular}[c]{@{}l@{}}Time\\(sec)\end{tabular}\\
		\hline
		\begin{tabular}[c]{@{}l@{}}Vehicle\\scale\end{tabular} & 9 & 9 &[-1,1] & 20 &102&306& $10^{-7}$& 401& 19311\\
		\hline
		Iris & 2&2& [0,10]&15 & 15&30& $10^{-2}$& 401& 1857\\
		\hline
		Breast & 20&10&[-10,10]& 5 & 15 &60 & $10^{-6}$& 201& 1896\\
		\hline
	\end{tabular}
}
\end{table}

\begin{table}[H]
	\centering
	\caption{SecureBoost Attack under Different Precisions}
	\label{table:differentprecision}
	\scalebox{0.76}{
	\begin{tabular}{|c|c|c|c|c|c|c|c|c|c|}
		\hline
		Data Set  
		&\multicolumn{3}{c|}{Vehicle scale}
		&\multicolumn{3}{c|}{Iris}   
		&\multicolumn{3}{c|}{Breast}\\
		\hline                  
		$\epsilon$ & $10^{-2}$ & $10^{-4}$& $10^{-6}$& $10^{-1}$& $10^{-4}$& $10^{-6}$& $10^{-1}$& $10^{-3}$& $10^{-6}$ \\
		\hline 
		Queries  & 102 &204&306&15&30&45&15&30&45\\
		\hline 
		Time (sec) & 6435 &12873 &19311&928&1857&2786&473&945&1420\\ 
		\hline
        
	\end{tabular}
}
\end{table}

\begin{figure}[H]
	\centering
	\includegraphics[width=0.9\linewidth]{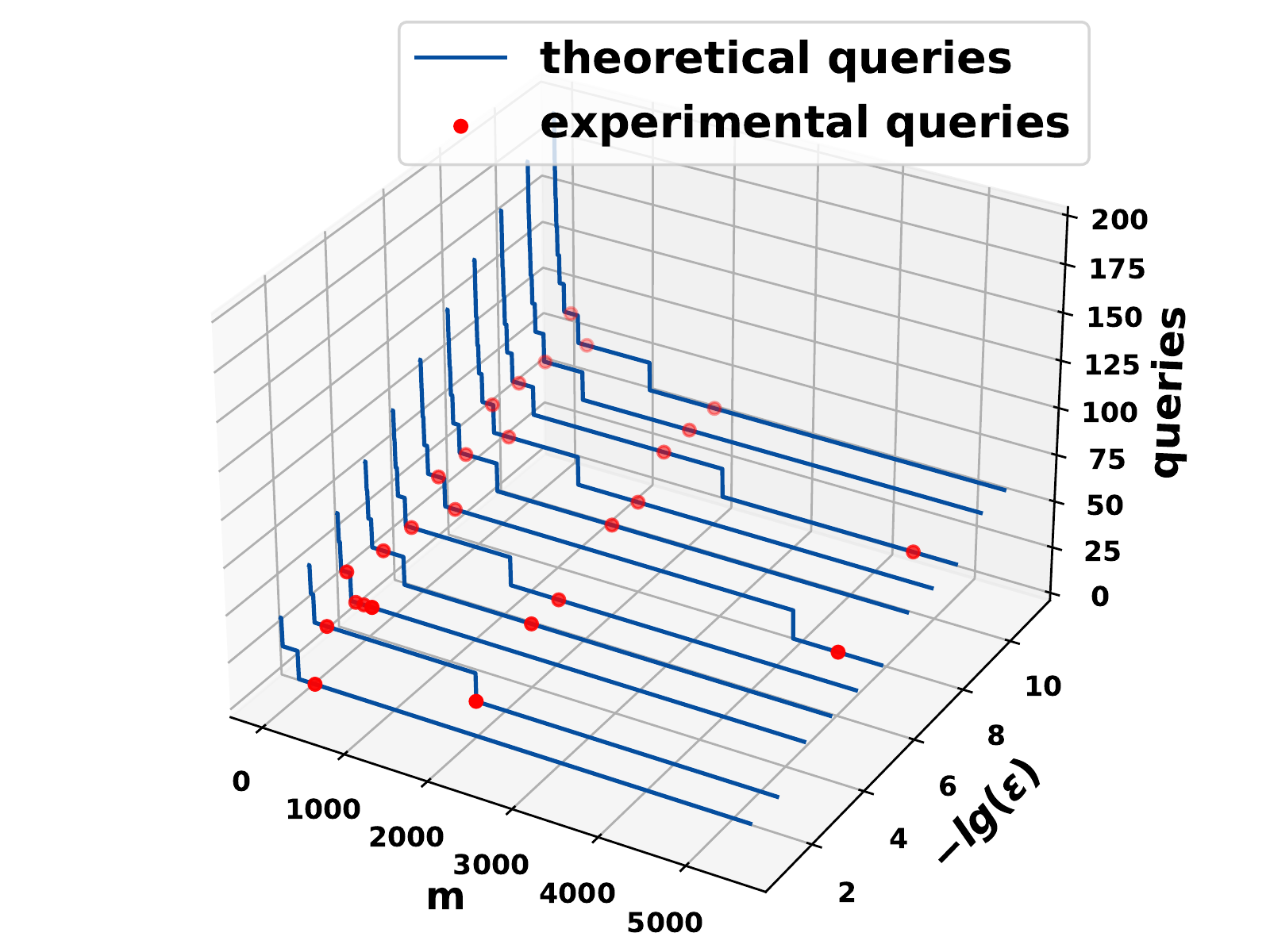}
	\caption{Breast Cancer dataset: the relation of the number of queries, precision and $m$.}
	\label{exsec}
\end{figure}

\subsection{Attacks to SecureBoost}
Since the attack to the multi-party SecureBoost can be reduced to the two-party case, we focus on the experiments of the latter.

Acting as an attacker, we train tree models respectively on three datasets, and perform attacks according to Alg.~\ref{alg:secureboostatk}. The results are shown in Table~\ref{table:secureboost}.  In the table, subtrees represent the number of subtrees used in the model. Victim nodes denotes the number of nodes that the attack targets at. Queries means the number of queries that the attacker submits and Records/Query stands for the number of records per query. $\epsilon$ denotes the precision required on victim's features.

We also compare the query numbers and running time of the attack for different precision $\epsilon$s in Table~\ref{table:differentprecision}, in which records per query is fixed to $401$. While we can fully recover the range of each training input, the higher the precision, the longer the running time for the attack. We also inverstigate how many queries that the attacker needs to invert the input range up to precision $\epsilon$ both theoretically and experimentally in Fig.~\ref{exsec}. $m$ denotes the number of input queries dividing the query interval in a uniform manner. The figure demonstrates that a higher precision (smaller $\epsilon$) typically requires a larger number of queries and under the same precision, the query number could be smaller given a larger $m$.

\section{Related Works}
Our work is mainly related to the following literature.

\subsection{Attacks to Federated Learning}
Many works are concerned about launching attacks in FL. Examples include \cite{hitaj2017deep, melis2019exploiting, wang2019beyond, nasr2019comprehensive, Bagdasaryan2018, baruch2019little}. Hitaj {\em et al.} \cite{hitaj2017deep} propose an attack which steals the private information of other workers in FL. The attacker recovers sensitive class representatives of other workers by taking advantage of the generative adversarial network. Melis {\em et al.} \cite{melis2019exploiting} and Wang {\em et al.} \cite{wang2019beyond} show that, beyond class representatives, the exchanged data in FL can leak too much information such as membership of a data record, unintended data properties, etc. Nasr {\em et al.} \cite{nasr2019comprehensive} design inference algorithms for both centralized and FL, and evaluate the white-box membership inference attacks to trace the training data records. Common in these works, the attacker can actively push stochastic gradient descent to leak more information about the participants' data. Different from their works, we analyze the privacy threats in a secure FL scenario where the message exchanged among different parties are encrypted, which largely confines the attacker's capability. Moreover, we do not assume the attackers actively alter the training of the global model, as such behavior would incur accuracy degradation and be detected.

Another line of works \cite{Bagdasaryan2018, baruch2019little} are about embedding backdoors in the FL models, which would be triggered later when the model is fed with the backdoor triggers. Our work is orthogonal to these works as the attacking goals are different: while their works try to manipulate the model, our attack targets at intercepting the private information of the participants.

\subsection{Secure Machine Learning}
Our work is related to secure machine learning where the model parameters, feature and data are protected under cryptography schemes. Some of the works use HE to secure the learning models \cite{han2019logistic, kim2018secure, kim2018logistic, gilad2016cryptonets, crawford2018doing}. For partially homomorphic encryptions, only addition and multiplication can be performed over the ciphertext, hence the neural network implements the activation function by approximated high-order polynomials instead. Both training and inference are conducted on the encrypted model and encrypted data. \cite{han2019logistic, kim2018logistic} present efficient algorithms for logistic regression on homomorphic encrypted data, and evaluate the algorithms on real-world applications. Leveled homomorphic encryption \cite{gilad2016cryptonets} and fully homomorphic encryption \cite{crawford2018doing} is also applied to learning models in but mainly as proofs of concepts.

A number of approaches have been proposed for secure FL, such as \cite{zhang2018gelu, cheng2019secureboost, ryffel2018generic} and libraries including {\tt FATE} \cite{fate2020}, {\tt PySyft} \cite{pysyft2020}, and {\tt TensorFlow Privacy} \cite{tensorflow2020}. These works consider privacy-preserving architectures for collaborative learning across different entities. GELU-Net \cite{zhang2018gelu} partitions a neural network to two non-colluding parties with one performing linear computation over encrypted data and the other executes non-polynomial computations on plaintext. SecureBoost \cite{cheng2019secureboost} presents a privacy-preserving approach to train a tree boosting model over multiple parties. {\tt FATE} implements secure computation protocols based on HE and MPC, whereas {\tt PySyft} relies on MPC, HE and differential privacy. {\tt TensorFlow Privacy} includes implementations of TensorFlow optimizers for training machine learning models with differential privacy. We mainly target at the secure FL based on HE and MPC in this paper.

\subsection{Comparison with Previous Works}
There are also some works focusing on the privacy of federated learning, especially on VFL. The methods and the difference between their works and ours are listed in the following.

In \cite{weng2020privacy}, Weng {\em et al.} propose reverse sum attack and reverse multiplication attack, which infer victim's private training data. The attacks share some similarity with ours in terms of demonstrating real-world secure VFL can potentially leak private training data. However, we give a more precise analysis on the threat that the participants face, by quantifying the amount of information required to invert all of victim's inputs. At its core, the reverse multiplication attack can be considered as only utilizing the linear constraints in inversion, whereas our attack incorporates quadratic and linear equations. Hence our attack is more powerful and better captures the privacy leakage. 


In \cite{luo2020feature}, the authors claim that their setting is the most stringent in that the adversary controls the trained vertical FL model as well as the model predictions. In our attack, the attacker only controls a part of the vertical FL model, the training labels, very few data and nothing else. We argue even under such a restrictive setting, the attacker can invert up to all victim's training data, posing as both significant and practical threats to FL participants.

\section{conclusion}

We reveal in this paper that privacy threats widely exist in today's secure federated learning framework, regardless the secure computation protocols implemented. This threat originates from the interaction between participants as well as some inappropriate training procedures. We analyze such privacy threats in linear/logistic regression models and SecureBoost tree models, and verify their impacts by launching attacks to real-world systems. Lessons can be learned that the secure FL systems in practice may not be privacy-presserving under some circumstances.

\ifCLASSOPTIONcaptionsoff
  \newpage
\fi

\bibliographystyle{IEEEtran}
\bibliography{main}

\begin{thebibliography}{10}
\providecommand{\url}[1]{#1}
\csname url@samestyle\endcsname
\providecommand{\newblock}{\relax}
\providecommand{\bibinfo}[2]{#2}
\providecommand{\BIBentrySTDinterwordspacing}{\spaceskip=0pt\relax}
\providecommand{\BIBentryALTinterwordstretchfactor}{4}
\providecommand{\BIBentryALTinterwordspacing}{\spaceskip=\fontdimen2\font plus
\BIBentryALTinterwordstretchfactor\fontdimen3\font minus
  \fontdimen4\font\relax}
\providecommand{\BIBforeignlanguage}[2]{{%
\expandafter\ifx\csname l@#1\endcsname\relax
\typeout{** WARNING: IEEEtran.bst: No hyphenation pattern has been}%
\typeout{** loaded for the language `#1'. Using the pattern for}%
\typeout{** the default language instead.}%
\else
\language=\csname l@#1\endcsname
\fi
#2}}
\providecommand{\BIBdecl}{\relax}
\BIBdecl

\bibitem{zhu2019deep}
L.~Zhu, Z.~Liu, and S.~Han, ``Deep leakage from gradients,'' in \emph{Advances
  in Neural Information Processing Systems}, 2019, pp. 14\,774--14\,784.

\bibitem{inproceedings1}
M.~Fredrikson, S.~Jha, and T.~Ristenpart, ``Model inversion attacks that
  exploit confidence information and basic countermeasures,'' 10 2015, pp.
  1322--1333.

\bibitem{hitaj2017deep}
B.~Hitaj, G.~Ateniese, and F.~Perez-Cruz, ``{Deep Models Under the GAN:
  Information Leakage from Collaborative Deep Learning},'' in \emph{Proceedings
  of the 2017 ACM SIGSAC Conference on Computer and Communications
  Security}.\hskip 1em plus 0.5em minus 0.4em\relax ACM, 2017, pp. 603--618.

\bibitem{nasr2019comprehensive}
M.~Nasr, R.~Shokri, and A.~Houmansadr, ``Comprehensive privacy analysis of deep
  learning: Passive and active white-box inference attacks against centralized
  and federated learning,'' in \emph{2019 IEEE Symposium on Security and
  Privacy (SP)}.\hskip 1em plus 0.5em minus 0.4em\relax IEEE, 2019, pp.
  739--753.

\bibitem{melis2019exploiting}
L.~Melis, C.~Song, E.~De~Cristofaro, and V.~Shmatikov, ``{Exploiting Unintended
  Feature Leakage in Collaborative Learning}.''\hskip 1em plus 0.5em minus
  0.4em\relax IEEE, 2019.

\bibitem{Chen_2016}
\BIBentryALTinterwordspacing
T.~Chen and C.~Guestrin, ``Xgboost,'' \emph{Proceedings of the 22nd ACM SIGKDD
  International Conference on Knowledge Discovery and Data Mining}, Aug 2016.
  [Online]. Available: \url{http://dx.doi.org/10.1145/2939672.2939785}
\BIBentrySTDinterwordspacing

\bibitem{wang2019beyond}
Z.~Wang, M.~Song, Z.~Zhang, Y.~Song, Q.~Wang, and H.~Qi, ``{Beyond Inferring
  Class Representatives: User-Level Privacy Leakage From Federated Learning},''
  in \emph{IEEE INFOCOM 2019-IEEE Conference on Computer Communications}.\hskip
  1em plus 0.5em minus 0.4em\relax IEEE, 2019, pp. 2512--2520.

\bibitem{Bagdasaryan2018}
E.~Bagdasaryan, A.~Veit, Y.~Hua, D.~Estrin, and V.~Shmatikov, ``{How to
  Backdoor Federated Learning},'' \emph{arXiv preprint arXiv:1807.00459}, 2018.

\bibitem{baruch2019little}
M.~Baruch, G.~Baruch, and Y.~Goldberg, ``{A Little Is Enough: Circumventing
  Defenses For Distributed Learning},'' \emph{arXiv preprint arXiv:1902.06156},
  2019.

\bibitem{han2019logistic}
K.~Han, S.~Hong, J.~H. Cheon, and D.~Park, ``Logistic regression on homomorphic
  encrypted data at scale,'' in \emph{Proceedings of the AAAI Conference on
  Artificial Intelligence}, vol.~33, 2019, pp. 9466--9471.

\bibitem{kim2018secure}
M.~Kim, Y.~Song, S.~Wang, Y.~Xia, and X.~Jiang, ``Secure logistic regression
  based on homomorphic encryption: Design and evaluation,'' \emph{JMIR medical
  informatics}, vol.~6, no.~2, p. e19, 2018.

\bibitem{kim2018logistic}
A.~Kim, Y.~Song, M.~Kim, K.~Lee, and J.~H. Cheon, ``Logistic regression model
  training based on the approximate homomorphic encryption,'' \emph{BMC medical
  genomics}, vol.~11, no.~4, p.~83, 2018.

\bibitem{gilad2016cryptonets}
R.~Gilad-Bachrach, N.~Dowlin, K.~Laine, K.~Lauter, M.~Naehrig, and J.~Wernsing,
  ``Cryptonets: Applying neural networks to encrypted data with high throughput
  and accuracy,'' in \emph{International Conference on Machine Learning}, 2016,
  pp. 201--210.

\bibitem{crawford2018doing}
J.~L. Crawford, C.~Gentry, S.~Halevi, D.~Platt, and V.~Shoup, ``Doing real work
  with fhe: the case of logistic regression,'' in \emph{Proceedings of the 6th
  Workshop on Encrypted Computing \& Applied Homomorphic Cryptography}, 2018,
  pp. 1--12.

\bibitem{zhang2018gelu}
Q.~Zhang, C.~Wang, H.~Wu, C.~Xin, and T.~V. Phuong, ``Gelu-net: A globally
  encrypted, locally unencrypted deep neural network for privacy-preserved
  learning.'' in \emph{IJCAI}, 2018, pp. 3933--3939.

\bibitem{cheng2019secureboost}
K.~Cheng, T.~Fan, Y.~Jin, Y.~Liu, T.~Chen, and Q.~Yang, ``Secureboost: A
  lossless federated learning framework,'' 2019.

\bibitem{ryffel2018generic}
T.~Ryffel, A.~Trask, M.~Dahl, B.~Wagner, J.~Mancuso, D.~Rueckert, and
  J.~Passerat-Palmbach, ``A generic framework for privacy preserving deep
  learning,'' 2018.

\bibitem{fate2020}
\BIBentryALTinterwordspacing
\emph{Federated Learning Algorithms In FATE}, 2020 (accessed August 16, 2020).
  [Online]. Available: \url{https://github.com/FederatedAI/FATE}
\BIBentrySTDinterwordspacing

\bibitem{pysyft2020}
\BIBentryALTinterwordspacing
\emph{PySyft}, 2020 (accessed August 16, 2020). [Online]. Available:
  \url{https://github.com/OpenMined/PySyft}
\BIBentrySTDinterwordspacing

\bibitem{tensorflow2020}
\BIBentryALTinterwordspacing
\emph{tensorflow}, 2020 (accessed August 16, 2020). [Online]. Available:
  \url{https://github.com/tensorflow/privacy}
\BIBentrySTDinterwordspacing

\bibitem{weng2020privacy}
H.~Weng, J.~Zhang, F.~Xue, T.~Wei, S.~Ji, and Z.~Zong, ``Privacy leakage of
  real-world vertical federated learning,'' \emph{arXiv preprint
  arXiv:2011.09290}, 2020.

\bibitem{luo2020feature}
X.~Luo, Y.~Wu, X.~Xiao, and B.~C. Ooi, ``Feature inference attack on model
  predictions in vertical federated learning,'' \emph{arXiv preprint
  arXiv:2010.10152}, 2020.

\end{thebibliography}

%







\begin{IEEEbiography}[{\includegraphics[width=25mm,height=35mm,clip,keepaspectratio]{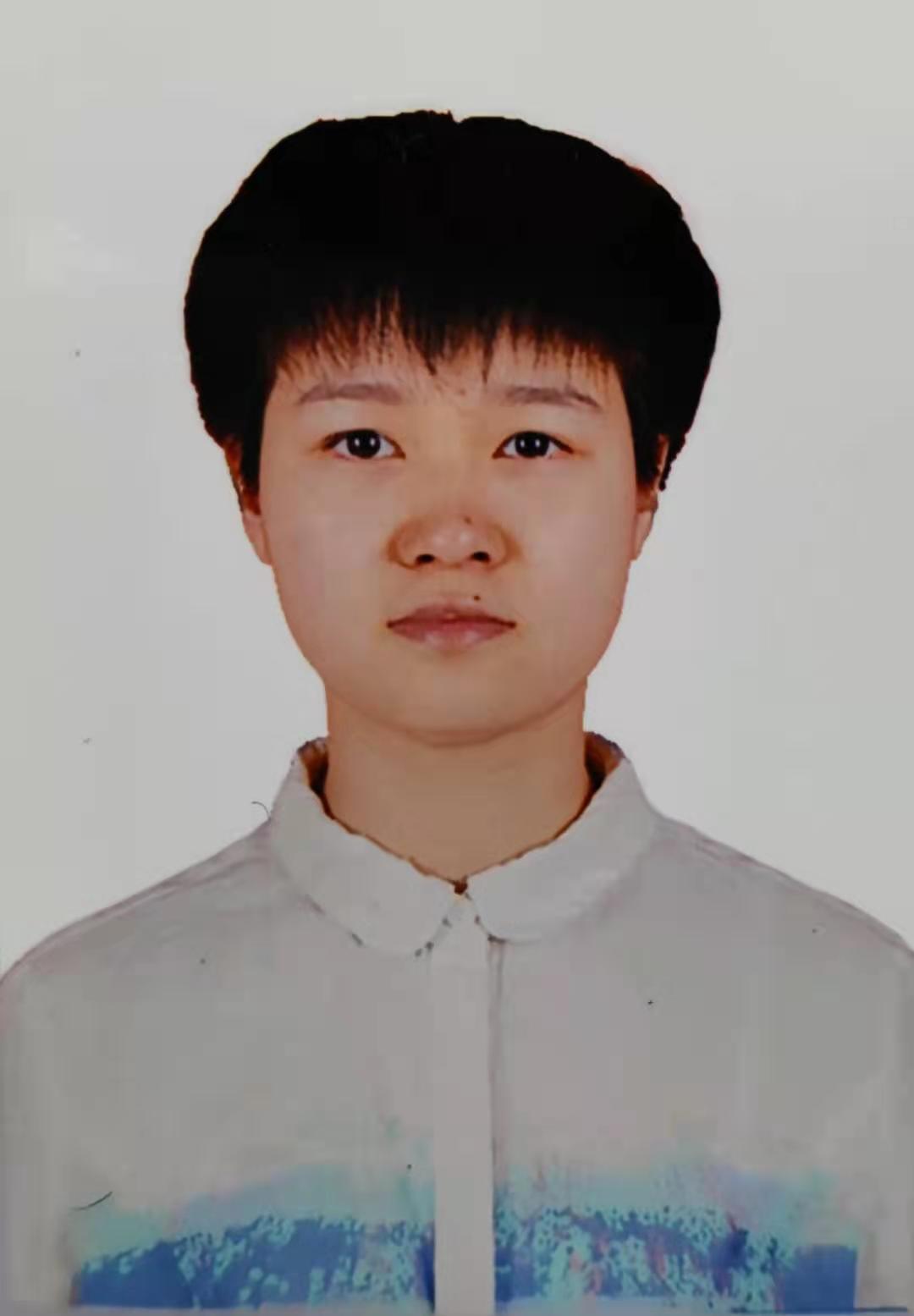}}]{Yuchen Li} has been a undergraduate student in Shanghai Jiao Tong University since 2018. Her research focuses on the intersected areas of security, privacy, machine learning. She is a student member of the IEEE.
\end{IEEEbiography}

\begin{IEEEbiography}[{\includegraphics[width=25mm,height=35mm,clip,keepaspectratio]{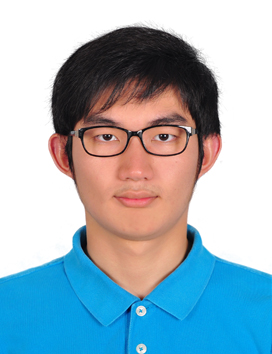}}]{Yifan Bao} enrolled in the Shanghai Jiao Tong University in 2019. As an undergraduate, he majors in Information Security in the School of Cyber Science and Engineering now. His research interests includes the intersected areas of security, privacy, and machine learning. 
\end{IEEEbiography}

\begin{IEEEbiography}[{\includegraphics[width=25mm,height=35mm,clip,keepaspectratio]{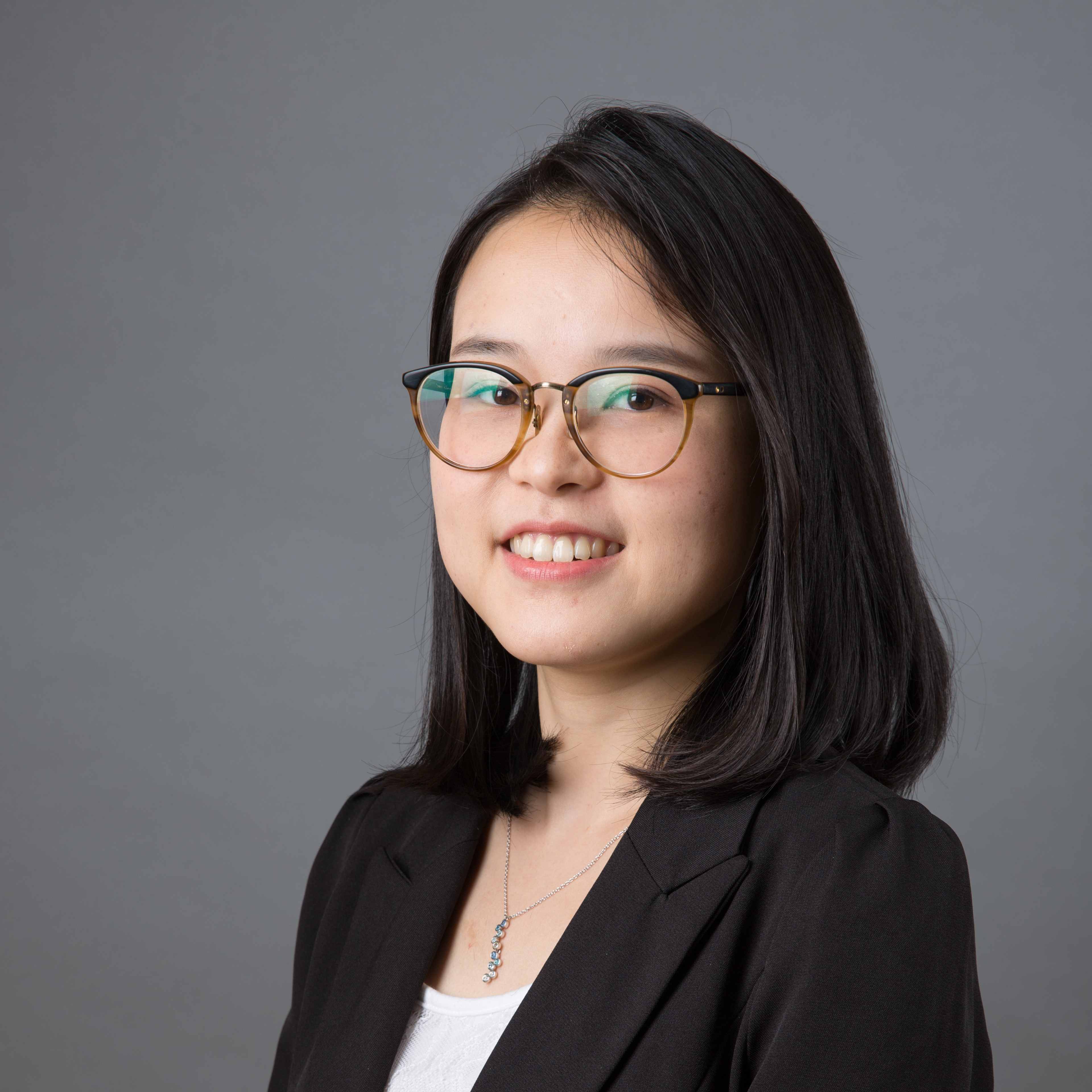}}]{Liyao Xiang} received the BS degree from the Department of Electronic Engineering at Shanghai Jiao Tong University in 2012, and the MS and PhD degrees in Department of Electrical and Computer Engineering from University of Toronto in 2015 and 2018 respectively. She is now an assistant professor at John Hopcroft Center for Computer Science of Shanghai Jiao Tong University. Her research interests includes the intersected areas of security, privacy, machine learning, and mobile computing. Topics include adversarial learning, privacy analysis in data mining, mobile systems and applications, mobile cloud computing. She is a member of the IEEE.
\end{IEEEbiography}

\begin{IEEEbiography}[{\includegraphics[width=25mm,height=35mm,clip,keepaspectratio]{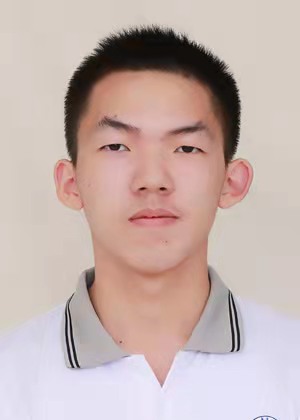}}]{Junhan Liu} started undergraduate studies in the Department of Computer Science and Engineering at Shanghai Jiao Tong University in 2019. His research interests includes the intersected areas of security, privacy and machine learning.
\end{IEEEbiography}

\begin{IEEEbiography}[{\includegraphics[width=25mm,height=35mm,clip,keepaspectratio]{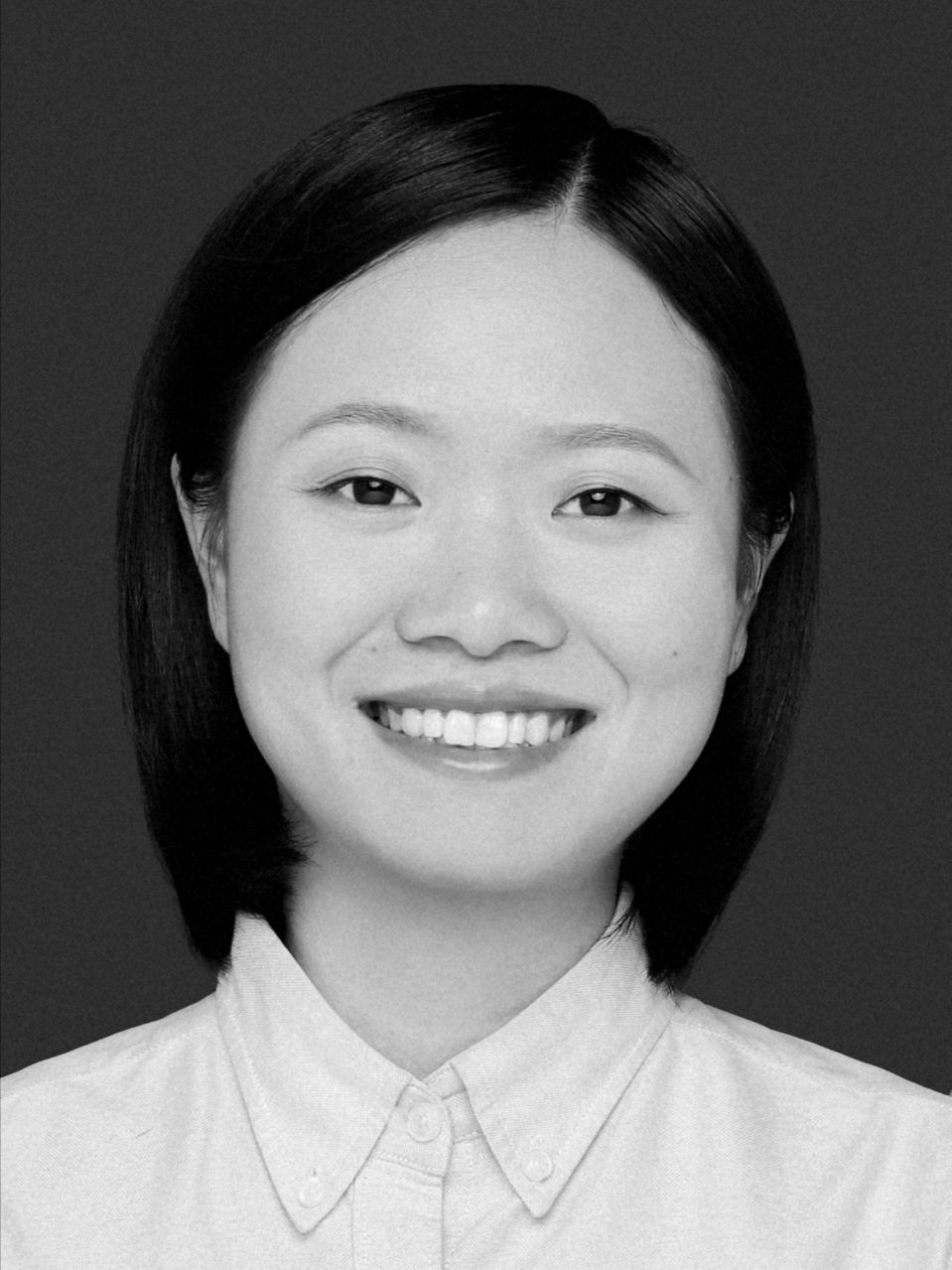}}]{Cen Chen} received her Ph.D. degree at the School of Information Systems, Singapore Management University in 2017. She was a visiting scholar in Heinz College, Carnegie Mellon University from 2015 to 2016. She is currently an algorithm expert with Ant Group. 
Her research interests include text mining,  privacy preserving machine learning and automated planning \& scheduling.
\end{IEEEbiography}

\begin{IEEEbiography}[{\includegraphics[width=25mm,height=35mm,clip,keepaspectratio]{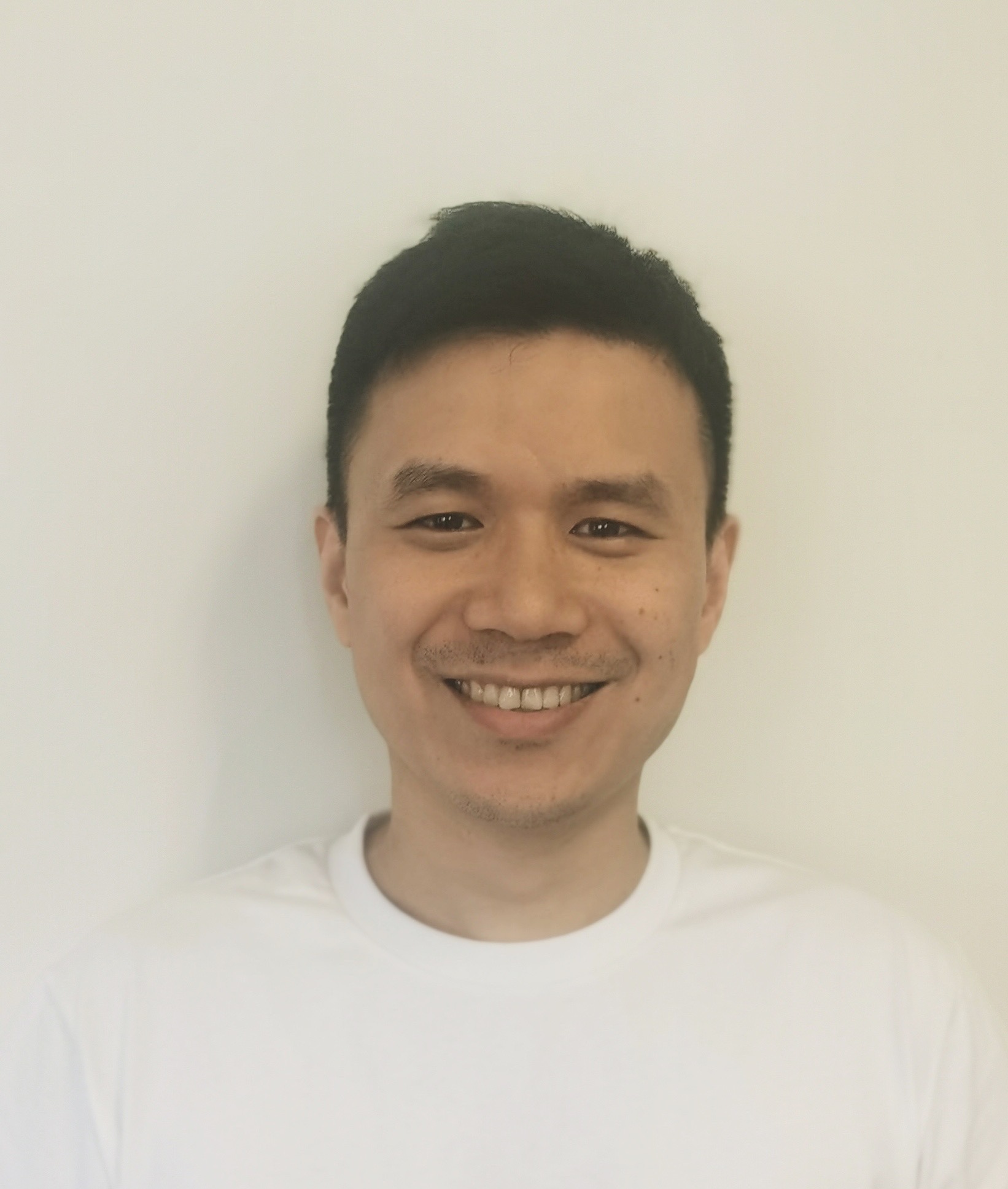}}]{Li Wang} is currently an Algorithm Expert at AI Department, Ant Group. He got his  master degree in Computer Science and Technology at Shanghai Jiao Tong University in 2010.  His research mainly focuses on privacy preserving machine learning, transfer learning, graph representation, and distributed machine learning.
\end{IEEEbiography}

\begin{IEEEbiography}
[{\includegraphics[width=1in,height=1.25in,clip,keepaspectratio]{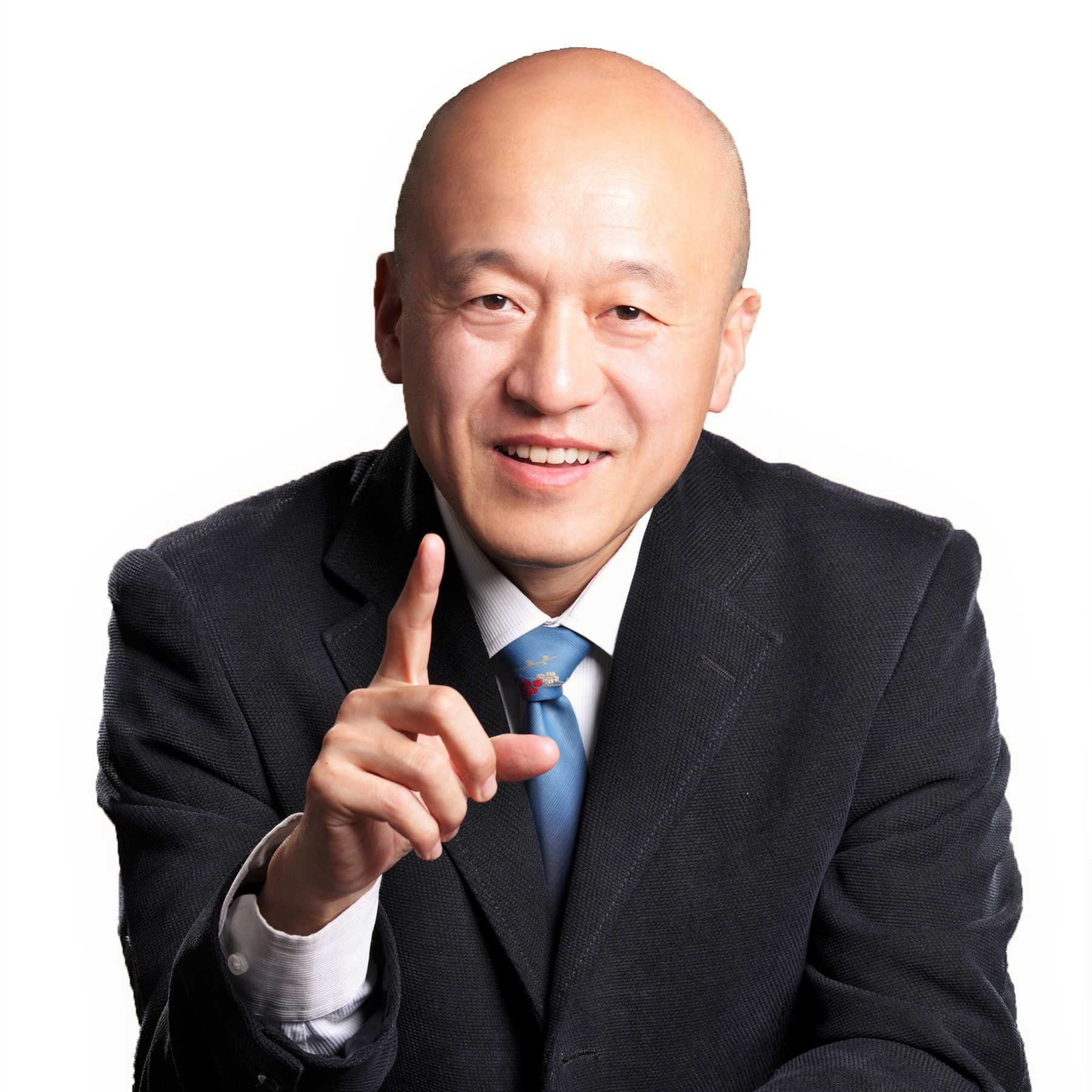}}]{Xinbing Wang} received the B.S. degree (with hons.) from the Department of Automation, Shanghai
Jiaotong University, Shanghai, China, in 1998, and
the M.S. degree from the Department of Computer Science and Technology, Tsinghua University,
Beijing, China, in 2001. He received the Ph.D.
degree, major in the Department of electrical and
Computer Engineering, minor in the Department
of Mathematics, North Carolina State University,
Raleigh, in 2006. Currently, he is a professor in
the Department of Electronic Engineering, Shanghai
Jiaotong University, Shanghai, China. Dr. Wang has been an associate editor
for IEEE/ACM Transactions on Networking and IEEE Transactions on Mobile
Computing, and the member of the Technical Program Committees of several
conferences including ACM MobiCom 2012, ACM MobiHoc 2012-2014,
IEEE INFOCOM 2009-2017.
\end{IEEEbiography}

\vfill

\end{document}